\pdfoutput=1
\documentclass[sigconf]{aamas} 

\usepackage{hyperref}
\usepackage{verbatim}
\usepackage{microtype}
\usepackage{enumitem}
\usepackage{booktabs} 
\usepackage{color,graphicx}
\usepackage{comment}
\usepackage{thmtools, thm-restate}
\usepackage{hyperref,xcolor}

\usepackage{amsmath,dsfont,bm,mathtools, bbm} 
\usepackage{subcaption}

\DeclareMathOperator*{\argmax}{argmax}

\newtheorem{theorem}{Theorem}

\newtheorem{observation}[theorem]{Observation}
\newtheorem{definition}{Definition}
\newtheorem{lemma}[theorem]{Lemma}
\newtheorem{claim}[theorem]{Claim}

\usepackage{amsthm}
\usepackage{multirow}
\usepackage{booktabs}
\usepackage{thm-restate}

\usepackage{adjustbox}
\usepackage{tikz}
\definecolor{babypink}{rgb}{0.96, 0.76, 0.76}
\definecolor{bubblegum}{rgb}{0.99, 0.76, 0.8}

\usepackage{algpseudocode}
\usepackage{graphicx}
\usepackage{soul}
\usepackage{algorithm}
\usepackage{caption}
\usepackage{xspace}

\newcommand{\groupReward}{Group-Merit RR}

\newcommand{\UCBGFa}{\textsc{UCB1 with GF-Constraint}\xspace}

\newcommand{\UCBGFb}{\textsc{Group-UCB1 with GF-Constraint}\xspace}

\newcommand{\UCBGFc}{\textsc{Group-UCB with Fair Exposure}\xspace}
\newcommand{\ourfair}{\texttt{Bi-Level Fairness}}

\newcommand{\gfair}{\texttt{Group Exposure Fairness}}
\newcommand{\mgfair}{\texttt{GEF}}
\newcommand{\ifair}{\texttt{Meritocratic Fairness}}
\newcommand{\mifair}{\texttt{MF}}\newcommand{\ouralgo}{\textsc{BF-UCB}}

\makeatletter
\newenvironment{subroutine}[1][htb]{%
    \renewcommand{\ALG@name}{Subroutine}
   \begin{algorithm}[#1]%
  }{\end{algorithm}}
\makeatother

\newcommand{\shweta}[1]{\textcolor{blue}{{#1}}}

\newcommand{\subham}[1]{\textcolor{black}{{#1}}}

\setcopyright{ifaamas}
\acmConference[AAMAS '24]{Proc.\@ of the 23nd International Conference
on Autonomous Agents and Multiagent Systems (AAMAS 2024)}{May 6, 2024 - May 10, 2024}
{Auckland, New Zealand}{A.~Ricci, W.~Yeoh, N.~Agmon, B.~An (eds.)}
\copyrightyear{2023}
\acmYear{2023}
\acmDOI{}
\acmPrice{}
\acmISBN{}

\acmSubmissionID{304}

\title{Simultaneously Achieving Group Exposure Fairness and Within-Group Meritocracy in Stochastic Bandits}

\author{Subham Pokhriyal}
\affiliation{
 \institution{Indian Institute of Technology Ropar}
 \city{Rupnagar}
 \country{India}}
\email{subham.22csz0002@iitrpr.ac.in}

\author{Shweta Jain}
\affiliation{
 \institution{Indian Institute of Technology Ropar}
 \city{Rupnagar}
 \country{India}}
\email{shwetajain@iitrpr.ac.in}

\author{Ganesh Ghalme}
\affiliation{
 \institution{Indian Institute of Technology Hyderabad}
 \city{Hyderabad}
 \country{India}}
\email{ganeshghalme@ai.iith.ac.in}

\author{Swapnil Dhamal}
\affiliation{
 \institution{Indian Institute of Technology Ropar}
 \city{Rupnagar}
 \country{India}}
\email{swapnil.dhamal@iitrpr.ac.in}

\author{Sujit Gujar}
\affiliation{
 \institution{International Institute of Information Technology, Hyderabad}
 \city{Hyderabad}
 \country{India}}
\email{sujit.gujar@iiit.ac.in}

\begin{abstract}
Existing approaches to fairness in stochastic multi-armed bandits (MAB) primarily focus on exposure guarantee to individual arms. When arms are naturally grouped by certain attribute(s), we propose \ourfair, which considers two levels of fairness. At the first level, \ourfair\ guarantees a certain minimum exposure to each group. 
To address the unbalanced allocation of pulls to individual arms within a group, we consider meritocratic fairness at the second level, which ensures that each arm is pulled according to its merit within the group. 
Our work shows that we can adapt a UCB-based algorithm to achieve a \ourfair\ by providing (i) anytime \gfair\
guarantees and (ii) ensuring individual-level \ifair\ within each group. We first show that one can decompose regret bounds into two components: (a) regret due to anytime group exposure fairness and (b) regret due to meritocratic fairness within each group. Our proposed algorithm \ouralgo\ balances these two regrets optimally to achieve the upper bound of   $O(\sqrt{T})$ on regret; $T$ being the stopping time. 
With the help of simulated experiments, we further show that \ouralgo\ achieves sub-linear regret; provides better group and individual exposure guarantees compared to existing algorithms; and does not result in a significant drop in reward with respect to UCB algorithm, which does not impose any fairness constraint.
\end{abstract}


\newcommand{\BibTeX}{\rm B\kern-.05em{\sc i\kern-.025em b}\kern-.08em\TeX}

\begin{document}

\pagestyle{fancy}
\fancyhead{}
\maketitle

\section{Introduction}
The conventional stochastic multi-armed bandit (MAB) problem considers the problem of a learner (or bandit) having a collection of arms, where each arm is associated with an unknown probability distribution governing the rewards. The objective is to devise an arm-selection strategy that optimizes the cumulative expected reward over a series of arm selections. Stochastic MABs find its use in a wide range of applications like sponsored search auctions~\cite{sharma2012truthful,kumar2020designing}, crowdsourcing~\cite{tran2014efficient,jain2014quality,jain2014multiarmed, biswas2015truthful,jain2018quality}, resource allocation~\cite{zuo2021combinatorial,10.5555/3545946.3598988,shweta2020multiarmed}, and many more.

This paper considers the problem of fair selection of arms in the stochastic multi-armed bandit problem. The fairness in stochastic MAB becomes important in applications where resources or opportunities are allocated over time among heterogeneous agents. In this context, each agent represents an arm, and pulling the arm corresponds to assigning a  resource/opportunity to the selected agent. An optimal policy, in this case, would end up providing the tasks to the most rewarding agents, leaving other arms with significantly less access to resources/opportunities. Therefore, it is crucial to devise a policy that ensures sufficient exposure to each agent.  

Current approaches towards fairness in stochastic MAB provide individual fairness guarantees to each arm by either offering \emph{minimum exposure} to each arm~\cite{patil2021achieving} or assuring \emph{meritocratic fairness}, i.e., ensuring that each arm is pulled in accordance with its merit (function of reward it generates)~\cite{wang2021fairness}. In many real-world applications, the number of arms is prohibitively large to guarantee exposure fairness at the level of individual arms. In such settings, the arms aka individual agents could be grouped based on certain attributes 
(e.g., gender, ethnicity, etc.)
which makes aggregate group-level fairness a more natural notion \cite{dwork2012fairness}. However, just ensuring group-level exposure fairness may lead to selecting only the best arm within each group. In summary, there is a need for an apt fairness notion. 

This paper introduces \ourfair\ (BF) in the Multi-Armed Bandit (MAB) problem. The first level of fairness guarantees minimum exposure to each group of arms. We call this notion \gfair, which stipulates that, at the end of each round of decision-making, an arm from each group must be selected or ``pulled'' for a minimum pre-defined fraction of times. 
Group fairness is particularly relevant in settings such as crowdsourcing, job screening, and college admissions, where each protected group is desired to be equitably represented~\cite{abbasi2021fair}. For instance, in a crowdsourcing setting where tasks need to be assigned to workers available on the platforms, the workers are naturally grouped into different groups, possibly based on gender or ethnicity. A crowdsourcing platform may be considered discriminatory if marginalized groups receive a much lesser number of tasks as opposed to the other groups. The group fairness notion, ensuring each group receives a minimum fraction of tasks, helps to mitigate this disparity.

A group fair policy, though fair at the group level, may still allocate resources/opportunities to within group individuals/arms in a skewed manner. That is, even within a group, it may disproportionately favor one arm and hence 
may not give enough opportunity to the arms within the group. In this paper, we consider the notion of  \ifair\  first proposed by \citet{wang2021fairness} to address the problem of within-group allocation guarantee to individual arms. Meritocratic fairness ensures that each arm within each group is pulled in proportion to its merit. \subham{For example, in the credit scoring problem\cite{babaei2023explainable}, a financial institute aims to determine the creditworthiness of potential borrowers. Here, each borrower acts as an arm and can be categorized into various groups based on a sensitive attribute (gender/marital-status/age). Each borrower's returns follow a probability distribution, which needs to be learned over time. A financial institute would like to diversify its lending amount across the different groups of borrowers, i.e., \gfair,\ while simultaneously ensuring that the amount is distributed in proportion to their financial capability, i.e., \ifair\ within the group.}

One way to achieve anytime \gfair\ guarantees while ensuring \ifair\ is to combine the algorithms in \cite{patil2021achieving} and \cite{wang2021fairness}. Both the above works proposed upper confidence bound (UCB) based algorithms. The algorithm presented in \cite{patil2021achieving} considers the minimum exposure guarantees to individual arms. One can extend this algorithm to ensure minimum exposure guarantees by applying the constraints to each group instead of each arm. The algorithm enforcing the minimum exposure constraint on each group will output a group to be pulled at each time. Once a group is selected, one can apply the algorithm presented in \cite{wang2021fairness}  to ensure meritocratic fairness within each group. Our proposed algorithm, Bi-Level Fair UCB, or \ouralgo\ in short,  is primarily motivated by the above approach. The main novelty of our work lies in providing the regret guarantees for \ouralgo .

The \emph{regret} of any online algorithm is defined as the difference in the reward obtained with the optimal algorithm and that with the online algorithm. 
The existing techniques from \cite{patil2021achieving,wang2021fairness} cannot be used to provide regret guarantees as the regret term becomes convoluted in terms of two fairness guarantees. 
We show that regret can be split into two terms, namely, regret due to the extra number of times a sub-optimal group is pulled and regret due to the learned fair policy within a group. Even after decomposing regret into two terms, there are two further challenges that need to be addressed to obtain the regret guarantee.  
First, the optimal policy in \cite{patil2021achieving} is defined with respect to the best individual arm; however, here, we have optimality with respect to the best group aka collection of arms. Therefore, existing regret proof techniques in \cite{patil2021achieving} that use UCB regret \cite{auer2010ucb} techniques will not work here. Since we are tackling group-level fairness, our regret requires bounding the number of pulls of the sub-optimal group as opposed to a single arm. Thus, our setting requires extending the regret in \cite{patil2021achieving} to combinatorial bandits setting \cite{chen2013combinatorial}. Second, the algorithm in \cite{wang2021fairness}\ assumes that the time horizon $T$ is known. However, since we provide meritocratic fairness within a group, this constraint of known time horizon would mean that the algorithm would know the number of times each group is pulled before the algorithm begins. This is not possible because the number of times a group will be pulled would depend on how learning progresses and what group fairness constraints were fed to the algorithm. 
We overcome these challenges and prove that the proposed algorithm \ouralgo\ provides sub-linear regret guarantees $O(\sqrt{T})$, $T$ being the arbitrary stopping time.  

In addition to theoretical analysis, the paper includes an empirical assessment of \ouralgo\ against conventional bandit algorithms and their fair variants. As baseline approaches, we consider the UCB algorithm without any fairness constraint, a group exposure fair algorithm by extending the algorithm in \cite{patil2021achieving} to groups, and the meritocratic fair algorithm in \cite{wang2021fairness}. In particular, we show that \ouralgo\ achieves sub-linear regret,
and that a simple extension of \cite{patil2021achieving} to group fairness may lead to biases within a group, while simple meritocratic fairness \cite{wang2021fairness} does not provide enough exposure to the groups. Our contributions can be summarized as follows.
\subsubsection*{Contributions}
\begin{enumerate}
\item We, for the first time, introduce the notion of \gfair\ in stochastic MABs.
\item We provide \ourfair\ notion in multi-armed bandits, which ensures not only group fairness but also meritocratic fairness within a group.
\item Inspired from UCB-based algorithms in \cite{patil2021achieving} and \cite{wang2021fairness}, we meld them perspicaciously to build \ouralgo. \ouralgo\ ensures \ourfair, i.e., it satisfies anytime group fairness constraint and learns meritocratic fair policy within groups.
\item We show that regret in our setting can be decomposed into two parts, allowing \ouralgo\ to achieve a regret of $O(\sqrt{T})$, where $T$ is the total number of rounds.
\item We finally validate our results via extensive experiments.
\end{enumerate}

\section{Related Work}
In the realm of multi-armed bandits (MAB), fairness has emerged as a significant concern. \citet{joseph2016fairness} introduced the concept of meritocratic fairness, ensuring that arms with higher rewards have a higher probability of being selected. \citet{liu2017calibrated} emphasized calibrated fairness, where arms are selected in proportion to their probability of being the best candidate, rather than based solely on average quality. \citet{gillen2018online} explored individual fairness, advocating for similar arms to be treated similarly in terms of selection probabilities. \citet{patil2021achieving} considered external constraints, designing algorithms that minimize regret while ensuring each arm is pulled a minimum fraction of rounds.  \citet{wang2021fairness} proposed Fair UCB and Fair Thompson Sampling algorithms, defining fairness regret based on the minimum merit of arms and the bounded Lipschitz constant of the merit function. All the above works ensure arm-level fairness, i.e., some exposure guarantees to each arm. So far, no works have tackled the issue of group fairness in a multi-armed bandit setting.

There have been some works in the setting known as group bandits, which categorize the arms into several groups. For example, \citet{jedor2019categorized} considered partial ordering over the groups and analyzed dominance among categories. \citet{wang2022max} introduced the idea of identifying groups with the highest mean reward for the worst arm.
\citet{gabillon2011multi} focused on the quality identification of arms within each bandit under a fixed budget constraint. \citet{scarlett2019overlapping} tackled best-arm identification in overlapping groups. While all the papers above focus only on pulling the optimal group in some sense, \citet{schumann2022group} addressed the potential biases in arm selection. In this context, fairness extends beyond the individual arm to group dynamics. This paper considers the case where pulling an arm from a particular group may inherently possess biases. They assume that, in general, the rewards of the groups are equal and try to mitigate the bias by learning the biases in each group. The paper does not consider any constraints required to pull from each group. Further, in real-world, the assumption of rewards coming from the same distribution for two groups may not hold. In addition to these works, contextual multi-armed bandits and clustering in multi-armed bandits have been widely explored with fairness considerations by \citet{chen2020fair,grazzi2022group}. Closer to our work is \cite{grazzi2022group}, where authors propose to provide exposure fairness according to the relative ordering of the arm, which is dependent on the group it belongs to. However, the paper does not consider any group fair exposure constraints.

When it comes to multi-agent, multi-armed bandit settings, agent-side fairness is emerging as an alternative perspective, where the goal is not merely to identify the best arm but to distribute the arms fairly among multiple agents \cite{liu2010distributed}. Concepts such as Nash welfare solutions \cite{hossain2021fair,barman2023fairness} have been developed to ensure fairness amongst agents. Our setting works in a single-agent, multi-armed bandit setting, and hence, we primarily focus on arm-side fairness.

\section{Model and Preliminaries}
A traditional stochastic multi-armed bandit (MAB) problem has a set of $n$ arms denoted as $\mathcal{A}$, where each arm $i$, when pulled, yields a reward following an unknown distribution with a mean reward of $\mu_i$. Initially, these mean rewards are concealed from the designer, and the primary objective is to learn these reward values within a specified time horizon denoted by $T$. In standard MAB algorithms, the central aim is to identify the optimal arm that generates the highest mean reward. 
 In our setting, the set of arms $\mathcal{A}$ is partitioned into $m$ groups, 
with $m << n$.
We denote the set of groups by $G$. The policy employed by the algorithm is denoted by $\pi = \{\pi^t\}_{t=1}^T$, where $\pi^t(i)$ denotes the probability of pulling an arm $i$ at time $t$ by the algorithm. Let $I_t \in g_t $ be the arm  that the learner pulls at round $t$ where $g_t \in G$ be the group pulled at round t.
 Let us denote the number of pulls for each arm $i$ till time $t$ as $N_{i,t}$ and for the group $g \in G$ as $N_{g,t}=\sum_{i\in g} N_{i,t}$. 
\subsection{\gfair}
Minimum pull guarantee for each arm was first introduced by  \citet{li2019combinatorial} with asymptotic guarantees, which was later extended to anytime fairness guarantee by  \citet{patil2021achieving}. In this work, the individual fairness constraints are exogenously specified by a pre-defined vector $\alpha = (\alpha_{1}, \alpha_{2},\ \ldots, \alpha_{n})$ such that $\sum_{i\in [n]}\alpha_i < 1$, with $\alpha_i$ denoting the minimum fraction of times arm $i$ needs to be pulled by the algorithm. This leads to the following definition:
\begin{definition}[\citet{patil2021achieving}]
\label{def:indfairness}
Given a fairness constraint vector $\alpha = (\alpha_i)_{i\in [n]} $, we call a strategy $\pi$ fair if $\mathbb{E}_\pi[N_{i,t}] \ge \lfloor\alpha_it\rfloor \;\;\; \forall i \in [n] \ \ \forall t \geq 1.$
\end{definition} 
We next extend the notion of fairness in Definition \ref{def:indfairness} to the group setting in the below definition. 
\begin{definition}[$\beta-$\gfair]
Let a given fairness constraint vector be $\beta = (\beta_g)_{g\in [m]} $ such that $\beta_g \in(0, \frac{1}{m}]$ for all $g \in G$ and $\sum_{g \in [m]} \beta_g < 1$. A policy $\pi$ 
is said to satisfy $\beta$-\gfair\ ($\beta$-\mgfair) if $\mathbb{E}_\pi [N_{g,t}] \geq \lfloor{\beta_gt}\rfloor\; \forall g \in G\;\; \forall t \geq 1.$
\end{definition}
As standard in the literature, we also assume, $\beta_g\le 1/m\:\forall g$. Note that, this extension to group-level fairness is inspired by a large body of work in the literature \cite{dwork2012fairness,abbasi2021fair} that focuses on equitable fairness across groups of individuals.  This aggregate guarantee is motivated by social justice and legal norms that require several protected groups to have sufficient access or exposure to opportunities and resources. Satisfying only group fairness may still lead to individual-level biases within a group, for example, by always pulling a single arm whenever a group is selected. To address this, we need to introduce equity fairness within the groups. To this, we now explain  \ifair\ within groups.

\subsection{Meritocratic Fairness within Groups}
While \mgfair\ ensures that each group of arms gets enough exposure, fair algorithms may still lead to a skewed distribution of opportunities within groups in favor of high-performing arms. We address this problem by imposing an additional constraint of \ifair\ (\mifair) within each group. \mifair\ ensures that each arm is pulled proportionately to its merit, defined by a merit function, and depends on the mean rewards. To define this fairness, we assume that there is a global merit function $f$ that maps true means to the merit values. This merit function is considered to be the same for all the arms and is provided as an input to the algorithm. Before we define \mifair, we first state the following assumption of Lipschitz continuity of $f$~\cite{wang2021fairness}.

\noindent\textbf{Assumption 1.} We assume that the merit function $f$ is Lipschitz  continuous, i.e., $|f(\mu_a) - f(\mu_a') \le L|\mu_a - \mu_a'|\ \forall \mu_a, \mu_a'$.


\noindent\textbf{Assumption 2} (Minimum merit assumption)\textbf{.} There exists $0 < \gamma_1 < \gamma_2 < \infty $ such that $0 < \gamma_1 \leq f(\mu) \leq  \gamma_2$ for all feasible expected rewards $\mu$.  

We can then define the \ifair\ within the group as follows. 
\begin{definition}[\ifair]
A policy $\pi_g^t(i)$ is said to satisfy \ifair\ iff
$\frac{\pi_g^t(i)}{\pi_g^t(j)} = \frac{f(\mu_i)}{f(\mu_j)}\ \forall i,j\in g$.
Here, $\pi_g^t(i)$ represents the probability of pulling an arm $i$ conditioned on the event that group $g$ is selected.
\end{definition}
\noindent The above definition is an extension of the definition in \cite{wang2021fairness} to the individual groups. Let $\pi_g^*$ represent a fair optimal policy, then it can be shown \cite{wang2021fairness} that  $\pi_g^*(i) = \frac{f(\mu_i)}{\sum_{j\in g} f(\mu_{j})}$.
Hence, if the $\mu$'s are known, the algorithm will follow $\pi_g^*$ for all rounds $t \in \{1,\ldots,T\}$. However, since the $\mu$'s are not known, the goal is to learn a policy $\pi^t_g$ which eventually converges to $\pi_g^*$ over a period of time. 

\subsection{\ourfair}
We now introduce \ourfair, which guarantees the fairness of exposure to arms as groups, and within a group, meritocratic fairness.

\begin{definition}[$\beta-$\ourfair]
\label{def:ourfair}
Given a fairness constraint vector $\beta = (\beta_g)_{g\in[m]}$, we say that a policy $\pi$ is said to satisfy $\beta-$\ourfair\ iff
\begin{enumerate}
\item $\pi$ satisfies $\beta-$\gfair , i.e., $\mathbb{E}_\pi[N_{g,t}] \ge \lfloor\beta_gt\rfloor\ \forall g\in G, \forall t \ge 1$, and
\item $\pi_g^t$ converges to $\pi_g^*$ for each group $g$, i.e., \\$\lim_{N_{g,T} \to \infty}\frac{1}{N_{g,T}}\sum_{t:g_t=g} \sum_{i\in g}|\pi^t_g(i) - \pi^*_g(i)| = 0\ \forall g\in G$.
\end{enumerate}
\end{definition}
\ourfair\ notion essentially ensures $\beta-$\gfair\ at group level and ensures that the group level policy converges to fair optimal group level policy. Let us now see how an optimal policy with the knowledge of $\mu$'s, maximizing the total reward while satisfying $\beta-$\ourfair , looks like. Since the probability of choosing an arm $i$ within a group $g$ is given by $\pi_g^*(i)$, the optimal group $g^*$ will be the one with the maximum expected reward, i.e.,
   $g^*=argmax_{g\in G}\left\{\sum_{i\in g} \frac{f(\mu_i)}{
\sum_{j\in g}f(\mu_j)}{\mu_i} \right\}$.
We begin by observing that in any optimal fair policy, a sub-optimal group gets precisely $\lfloor\beta_g T\rfloor$ pulls, whereas the optimal group is pulled the remaining number of times; this leads to the following simple proposition.

\begin{observation} \label{obs: obs1}
A policy $\pi^*$ satisfying $\beta$-\ourfair\  is said to be optimal 
 iff
it satisfies the following conditions at all time instances $t$:
\begin{enumerate}
\item For all $g \neq g^*$ such that $\beta_g = 0$, we have $N_{g,t} =0 $. That is,  $N_{i,t} = 0$ for all $i \in g$.
\item For all $g \neq g^*$ such that $\beta_g > 0$, we have  
 $N_{g,t} = \lfloor \beta_gt\rfloor   $  and 
$\pi^*_g(i) = \frac{f(\mu_i)}{\sum_{j\in g}f(\mu_j)}$.
\item 
$N_{g^*,t}= t - \sum_{g \neq g^*} \lfloor \beta_gt\rfloor $  and 
$\pi^*_{g^*}(i) = \frac{f(\mu_i)}{\sum_{j\in g^*}f(\mu_j)}$.
\end{enumerate}
\end{observation}


The performance of any online policy is measured by its \emph{regret} -- the difference in the reward obtained by the optimal policy and that by the online policy. In order to find the regret, let us first find the reward by the optimal policy $\pi^*$ which is given as:

\vspace{-2.15mm}
\begin{equation}
\label{eq:finalregret}
\begin{aligned}
 R_{\beta}^*(T) &=  \sum_{g\in G} \lfloor\beta_{g}T\rfloor\left(\sum_{i\in g} \frac{f(\mu_i)}{
\sum_{j\in g}f(\mu_j)}{\mu_i} \right)\\ &+  \left(T-\sum_{g\in G}\lfloor\beta_{g}T\rfloor\right)\left(\sum_{i\in g^*} \frac{f(\mu_i)}{
\sum_{j\in g^*}f(\mu_j)}{\mu_i} \right)
\end{aligned}
\end{equation}

\noindent
We will assume that $g^*$ is unique for ease of explanation. However, this is not a necessary assumption for the regret guarantees to hold. We now define regret for a policy satisfying  \ourfair.
\begin{definition}
\label{def:main-regret}
Given a fairness constraint $\beta_g,\;$for all $\; g \in G$,  the regret of a policy $\pi$ satisfying \ourfair\ is defined as:

\begin{equation}
 \mathfrak{R}_{\pi}^\beta (T)=  R_{\beta}^*(T) -\sum_{g\in G}\sum_{i\in g}\mathbb{E}_\pi[N_{i,T}]\mu_i
\end{equation}
\end{definition}


In Section \ref{sec:theoretical-section}, we will show that the regret can be decomposed into two parts: (i) regret due to extra pull of a non-optimal group and (ii) regret due to suboptimal learning of policy within each group. We now propose \ouralgo\ in the next section, an upper confidence bound (UCB) based algorithm, satisfying \ourfair.

\section{\ouralgo: Proposed Algorithm}

In this section, we propose our algorithm that ensures group exposure fairness (\mgfair) guarantees while maintaining meritocratic fairness (\mifair) within a group. The detailed algorithm is presented in Algorithm \ref{alg:group-exposurefair}. As a standard practice in any MAB algorithm, our algorithm starts by pulling each arm once to get some estimates of  $\mu_i$'s $\forall i\in N$. Note that a simple round-robin arm-pulling strategy breaks \mgfair\ if some of the groups have a lot more arms than other groups. In order to prevent this, we use the fact that each $\beta_g \le 1/m$, and therefore, we select the groups in round-robin fashion until each arm in each group is pulled at least once. This is depicted in line numbers \ref{line3a} to \ref{line13a} of Algorithm \ref{alg:group-exposurefair}. If, for a group, all the arms are completely exhausted, we start pulling the arms based on maintaining exposure fairness (line number \ref{line10a} of Algorithm \ref{alg:group-exposurefair}). 

Once all the arms are pulled at least once, the algorithm (i) first selects a group from which arm is to be pulled (Group Selection Strategy) and then (ii) chooses the arm to pull within the group (Arm within Group Selection Strategy).

\subsection*{Group Selection Strategy}
Motivated from \cite{patil2021achieving}, we propose an algorithm that provides anytime \mgfair\ guarantees. As described above,  \emph{Initialization Phase} does not violate \mgfair. For the remaining rounds, we use a similar approach as used by \citet{patil2021achieving}, but on the groups instead of arms. At each time $t$, the algorithm maintains a set $UFG\_Set(t)$, which denotes the set of groups that are on the verge of violating \mgfair\ (Line \ref{line16a}). If at all there exists a group in $UFG\_Set(t)$, Algorithm \ref{alg:group-exposurefair} selects a group $g \in argmax \, \beta_g(t-1) - N_{g,{t-1}}$ to ensure group fairness in the next round. If $UFG\_Set(t)$ is empty, the idea is to select the group with the maximum expected reward. Since the maximum expected reward is unknown beforehand, the function $Learn(\cdot)$ returns the group which helps in learning these estimates better. One could use the $Learn(\cdot)$ function based on Upper Confidence Bound (UCB) based algorithm \cite{lai1985asymptotically,auer2002using} or Thompson sampling-based algorithms \cite{thompson1933likelihood}. For completeness, we have given a UCB-based algorithm in Subroutine \ref{alg:learn}. Theorem \ref{thm:correctnessOne} in Section \ref{sec:theoretical-section} shall show that the algorithm satisfies \gfair. 

\subsection*{Arm within Group Selection Strategy}
Once the group is chosen, the algorithm's
arm selection strategy basically selects the arm based on \gfair. Our Exposure subroutine, given in Subroutine \ref{alg:exposure}, looks similar to the algorithm provided by \citet{wang2021fairness} with one key distinction. The algorithm in \cite{wang2021fairness} assumes that $T$ is known to the algorithm. Since we aim to ensure exposure fairness within each group, each group $g$ is not chosen $T$ number of times but is chosen $N_{g,T}$ number of rounds, which is a random variable. Therefore, we must design an algorithm without information about how often a group is selected. To tackle this challenge, we replace parameter $w_0$ with the value $\sqrt{2\ln (4N_{g,t}k_g/\delta)}$ instead of $\sqrt{2\ln (4TK/\delta)}$ in \cite{wang2021fairness}. Here, $k_g = |g|$ denotes the number of arms in group $g$. In the next section, we prove that this change still provides sub-linear regret guarantees without knowledge of $T$. Theorem \ref{thm:correctnessMF} in Section \ref{sec:theoretical-section} shall show that the algorithm satisfies \ifair.

\begin{algorithm}[t!]
\caption{\ouralgo}
\label{alg:group-exposurefair}
\begin{algorithmic}[1]
\Require {$[n]$, Group Partition $G = \{g_1, \ldots, g_m\}$, Fairness parameter $\{\beta_{g}\}_{g\in G}$, Learning Function $Learn(.)$}
\\
\textbf{\texttt{Initialization Phase}}
\State {$N_{g,0}=0 \;\;\forall g \in G$, where $N_{g,t}$ is the number of times a group is chosen till time $t$} 
\State {$S_{i,0}=0 \;\;\forall i \in [n]$, where $S_{i,t}$ denotes the reward of arm $i$ till time $t$}   \label{line3a}
\State $max_{size}= \argmax_{j\in [1, \ldots, m]} |g_j|$
\For{$k \in [1,... ,max_{size}]$}
\For{$g \in [1,... ,m]$}
\If{$\exists i\in g$ such that $N_{i,t} = 0$}
\State Pull arm $I_t= i$
\Else    
\State $I_t=Exposure(g,t,f,\{N_{i,t}\}_{i\in g},\{S_{i,t}\}_{i \in g})$    \label{line10a}
\EndIf
\State Update $N_{I_t,t} = N_{I_t,t} + 1$, $N_{g,t} = N_{g,t}+1$, and update $S_{I_t,t}$ 

\hspace{4mm} based on reward
\EndFor          \label{line13a}
\EndFor
\State $t_{init}= m \cdot max_{size}$                     

\For{$t \in [t_{init}+1, ... ,T]$}                     \label{line16a}

\State $UFG\_Set(t) = \{g \; |\; \beta_{g}(t-1) -N_{g,t-1}> 0\}$

\If {$UFG\_Set(t) \neq \phi$}

\State $g = \argmax_{k\in UFG\_Set(t) }\{\beta_{k}(t-1)-N_{k,t-1}\}  $
\State $I_t=Exposure(g,t,f,\{N_{i,t}\}_{i\in g},\{S_{i,t}\}_{i \in g})$
\Else

\State $g = Learn (G,t,f,\{N_{i,t}\}_{i\in [n]},\{S_{i,t}\}_{i \in [n]})$
\State $I_t=Exposure(g,t,f,\{N_{i,t}\}_{i\in g},\{S_{i,t}\}_{i \in g})$
\EndIf
\State Update $N_{I_t,t} = N_{I_t,t} + 1$, $N_{g,t} = N_{g,t}+1$ and update $S_{I_t,t}$ 

 based on reward
\EndFor    
\end{algorithmic}
\end{algorithm}

\begin{subroutine}
\caption{Exposure$(g,t,f,\{N_{i,t}\}_{i\in g},\{S_{i,t}\}_{i \in g})$}
\label{alg:exposure}
\begin{algorithmic}[1]
\State   $ \hat{\mu}_{i,t}= \frac{S_{i,t}}{N_{i,t}}, \forall \;i \in g$
\State $w_t^g =\sqrt{2\ln(4N_{g,t}{k_{g}}/\delta)}$
\State   $w_{i,t}=\frac{w_t^g}{\sqrt{N_{i,t}}}, \forall \;i \in g$
\State $CR_t={(\mu: \forall \;i \in g,\;  \mu_i \in [ \hat{\mu}-w_{i,t} , \hat{\mu}+w_{i,t} ] ) } $
\State $\Tilde{\mu}_t^g= \argmax_{\mu \in {CR_t}}\sum_{i\in g}\frac{f({\mu}_i)}{\sum_{i'\in g}{f({\mu}_{i'})}} \mu_i$
\State $\pi_{i,t}=\frac{f({\Tilde{\mu}_{i,t}})}{\sum_{i'\in g}{f(\Tilde{\mu}_{i',t}})}$, $\forall\; i \in g$
\State $I_t \sim \pi_t$
\State \textbf{Return {$I_t$}}
\end{algorithmic}
\end{subroutine}

\begin{subroutine}
\caption{Learn$(G,t,f,\{N_{i,t}\}_{i\in [n]},\{S_{i,t}\}_{i \in [n]})$}
\label{alg:learn}
\begin{algorithmic}[1]
\State   $ \hat{\mu}_{i,t}= \frac{S_{i,t}}{N_{i,t}}, \forall\;i \in [n]$
\State $w^g_t =\sqrt{2\ln(4N_{g,t}{k_{g}}/\delta)}\;, \forall\;g \in G$
\State   $w_{i,t}=\frac{w^g_t}{\sqrt{N_{i,t}}}, \forall\;i \in g, \forall\;g \in G$
\State $CR^g_t={(\mu: \forall\; i\in g,\; \mu_i \in [ \hat{\mu}_i-w_{i,t} , \hat{\mu}_i+w_{i,t} ] ) }, \forall\; g \in G$
\State $\Tilde{\mu}^g_t= \argmax_{\mu \in {CR^g_t}}\sum_{i\in g}\frac{f(\mu_i)}{\sum_{i'\in g}{f(\mu_{i'})}} \mu_i, \forall\; g \in G$
\State $j = \argmax_{g \in G}\sum_{i \in g}\frac{f({\Tilde{\mu}^g_{i,t}})}{\sum_{i'\in g}{f(\Tilde{\mu}^g_{i',t}})} {\Tilde{\mu}^g_{i,t}}$
\State \textbf{Return {$j$}}
\end{algorithmic}
\end{subroutine}

\section{Theoretical Results}
\label{sec:theoretical-section}

This section presents three main results of the paper, namely, 
\begin{enumerate}[leftmargin=*]
\item \emph{\ourfair\ Guarantees of \ouralgo}: The policy output by Algorithm \ref{alg:group-exposurefair} satisfies $\beta-$\ourfair\ (Definition \ref{def:ourfair}). 
\item \emph{Regret Decomposition Result:} The regret in Definition \ref{def:main-regret} can be decomposed into two parts, namely, \gfair\ regret and \ifair\ regret.
\item \emph{Sub-linear Regret:} The regret achieved by our algorithm is $O(\sqrt{NT})$.
\end{enumerate}

\subsection{\ourfair\ Guarantees of \ouralgo}
We show that \ouralgo\ satisfies \ourfair, in two parts. First, it satisfies \mgfair, and second, it satisfies \mifair. 
\begin{theorem}\label{thm:correctnessOne}
Algorithm \ref{alg:group-exposurefair} satisfies anytime \mgfair\ guarantees, i.e., $\lfloor \beta_g t \rfloor \le N_{g,t}$ for all $t \ge 1$ and for all groups $g \in G$. 
We have $\beta_g>0$ and for any $\beta-$\ourfair\ 
algorithm  $\beta_g \in (0,\frac{1}{m}]$ for all $ g\in[m]$ and $\sum_{g\in m}\beta_g<1$.
\end{theorem}
\begin{proof}

\subham{
Let $t_{init} = m \cdot max_{size}$, and up to this round, each group is pulled in a round-robin fashion. Consequently, for all groups $g \in G$, the number of times group g is pulled, denoted as $N_{g,t}$, satisfies the inequality $N_{g,t} \ge \lfloor t/m \rfloor \ge \beta_gt$. The last inequality is derived from the fact that for all groups $g \in G,\ \beta_g \le 1/m$.
}
\\
\subham{
For all $t \ge t_{init}$, the correctness proof follows analogous steps as outlined in \cite{patil2021achieving} by establishing a mapping between each group in our setting and an arm in their setting.}
\end{proof}
\begin{theorem} 
\label{thm:correctnessMF}
Algorithm \ref{alg:group-exposurefair} satisfies \mifair , i.e., $$\lim_{N_{g,T} \to \infty}\frac{1}{N_{g,T}}\sum_{t:g_t = g} \sum_{i\in g}|\pi^t_g(i) - \pi^*_g(i)| = 0\ \forall g\in G.$$
\end{theorem}
\begin{proof}
Let us consider a set $T_g = \{t_1, t_2, \ldots, t_{N_{g,T}}\}$ which denotes the time steps when the group $g$ is pulled. It is easy to see from Hoeffding's inequality \cite{hoeffding1963prob} that,
\begin{equation}
\label{eq:upperboundkg}
\mathbb{P}(\mu_i\in CR_t) \ge 1-\frac{\delta^2}{8N_{g,t}^2k_g}\ \forall t>t_{k_g},i\in[g] 
\end{equation}
We also know that the sequence 
$\sqrt{1/N_{i_t,t}}-
\mathbb{E}_{i \in \pi_g^t}\sqrt{1/N_{i,t}}$
is a martingale difference sequence $\forall t>t_{k_g}$. Thus, we have\\
$\left\lvert\sqrt{1/N_{i_t,t}}-
\mathbb{E}_{i \in \pi_g^t}\sqrt{1/N_{i,t}}
\right\rvert\leq 1$.
We can apply the Azuma-Hoeffding's inequality to get that with probability at least $1-\delta/2$,
\begin{equation}
\label{eq:martingale}
\left\lvert
\sum_{t\in T_g}\mathbb{E}_{i \in \pi_g^t}\sqrt{1/N_{i,t}} - \sum_{t \in T_g}\sqrt{1/N_{i_t,t}}
\right\rvert
\leq \sqrt{2N_{g,t}\ln(4/\delta)}
\end{equation}


Thus, for any group $g$, we have:
\allowdisplaybreaks
\begin{align*}
&\sum_{t\in T_g}\sum_{i\in g}|\pi_g^{*}(i) -\pi_g^t(i)| \leq\sum_{t\in T_g}\frac{2\sum_{i\in g}\frac{f(\tilde{\mu}_{i,t})}{f(\tilde{\mu}_{i,t})}|f(\tilde{\mu}_{i,t}) - f(\mu_i)|}{\sum_{j\in g}f(\tilde{\mu}_{i,t})} \tag{By following steps of proof of Theorem 3.2.1 from \cite{wang2021fairness}}\\
&\leq \sum_{t\in T_g}\frac{2L\sum_{i\sim \pi_g^t}|\tilde{\mu}_{i,t} - \mu_i|}{\gamma_1}\tag*{(From Assumptions 1 and 2)}\\
&\le \sum_{t\in T_g}\frac{2L\sum_{i\sim \pi_g^t}w_{i,t}}{\gamma_1},\;\;\;\text{w.p. } \left(1-\frac{\delta^2}{8N_{g,t}^2k_g}\right) \tag{From Equation~(\ref{eq:upperboundkg})}\\
&\le \sum_{t\in T_g}\frac{2L\sum_{i\sim \pi_g^t}\sqrt{ \frac{2\ln(4k_gN_{g,t}/\delta)}{N_{i,t}}}}{\gamma_1},\;\;\;\text{w.p. } \left(1-\frac{\delta^2}{8N_{g,t}^2k_g}\right)\\
&\le \frac{2L\sqrt{2\ln(4k_gN_{g,T}/\delta)}}{\gamma_1}\sum_{t\in T_g}\mathbb{E}_{i \in \pi_g^t}\sqrt{ \frac{1}{N_{i,t}}} \tag{$\because N_{g,t} \le N_{g,T}$}\\
&\le \frac{2L\sqrt{2\ln(4k_gN_{g,T}/\delta)}}{\gamma_1}\left(\sqrt{2N_{g,T}\ln(4/\delta)}+ \sum_{t\in T_g, t\ge t_{k_g}} \sqrt{ \frac{1}{N_{i_t,t}}}\right)\\
&\le \frac{2L\sqrt{2\ln(4k_gN_{g,T}/\delta)}}{\gamma_1}\left(\sqrt{2N_{g,T}\ln(4/\delta)}+ 2\sqrt{N_{g,T}k_g}\right)
\end{align*}
The last inequality follows from AM-GM inequality. The fact that $\sum_{t:g_t = g} \sum_{i\in g}|\pi^t_g(i) - \pi^*_g(i)|$ is sub-linear in $N_{g,T}$ completes the proof.
\end{proof}
It is to be noted that $\sum_{t\in T_g}\sum_{i\in g}|\pi_g^{*}(i) -\pi_g^t(i)|$ is also referred to as fairness regret $FR_T$ in \cite{wang2021fairness}. Theorem \ref{thm:correctnessMF}
says that fairness regret due to \ifair\ is sublinear, i.e., $O(\sqrt{T})$. The fairness regret due to \gfair\ will be zero since we provide anytime \gfair\ guarantees.

\subsection{Regret Decomposition Theorem}
Our next result shows that the regret of any algorithm satisfying \ourfair\ can be decomposed into \mgfair\ regret and \mifair\ regret.  Let $R_{g}^* = \sum_{i\in g} \pi_{g}^*(i)\mu_i$ denote the optimal expected reward of group $g$. Further, define  $R_g^{t} = \sum_{i\in g} \pi_{g}^t(i)\mu_i$ to be the expected reward generated from policy $\pi_{g}^t$. Also, $\Delta_g = R_{g^{*}}^* - R_g^* \text{   and  } \Delta_g^t = R_g^* - R_g^t$. Then, we have the following theorem.

\begin{theorem}[Regret decomposition Theorem]
\label{thm:regret-decomposition}
The reward regret, $\mathfrak{R}_\pi^\beta (T)$, can be decomposed into two parts, namely, the regret due to extra pull of non-optimal group and the regret due to suboptimal learning of policy within each group, i.e.,
\begin{equation}
\label{eq:RegretDecomposition}
\mathfrak{R}_\pi^\beta (T)= \sum_{g\in G} \left(\mathbb{E}_\pi [N_{g,T}] - \lfloor\beta_gT\rfloor\right)\Delta_g + \sum_{t=1}^T \sum_{g\in G} \mathbbm{1} \left( g_t = g \right) \Delta_g^t. \end{equation} Here, $g_t$ denotes the group that is selected by the algorithm at time $t$.
\end{theorem}

\begin{proof}
\allowdisplaybreaks
\begin{align*}
\mathfrak{R}_\pi^\beta(T) &= \sum_{g\in G} \lfloor\beta_{g}T\rfloor R_g^*+ \left(T-\sum_{g\in G} \lfloor\beta_{g}T\rfloor\right)R^*_{g^*} - 
\sum_{g\in G}\sum_{i\in g}\mathbb{E}_\pi[N_{i,T}]\mu_i \tag{From Equation (\ref{eq:finalregret})}\\
 &= \sum_{g\in G} \lfloor\beta_{g}T\rfloor R_g^*+ \left(T-\sum_{g\in G} \lfloor\beta_{g}T\rfloor\right) R^*_{g^*} - \sum_{t \in T} \sum_{g\in G}\mathbbm{1} (g_t=g) R_g^t \tag{By the definition of $R_g^t$} \\
&= TR_{g^*}^{*} -\sum_{g\in G} \lfloor\beta_{g}{T}\rfloor \left(  {R_{g^*}^{*} - R_g^*} \right)- 
\sum_{t \in T} \sum_{g\in G}\mathbbm{1} (g_t=g){R_g^t}\tag{Rearranging terms}  \\
&= TR_{g^*}^{*} -\sum_{g\in G} \lfloor\beta_{g}{T}\rfloor  {\Delta_g}  - 
\sum_{t \in T} \sum_{g\in G}\mathbbm{1} (g_t=g){R_g^t}  
\end{align*}
From the definition of $\Delta_g$ and $\Delta_g^t$, we have, $R_g^t = R_g^* - \Delta_g^t = R_{g^*} - \Delta_g -\Delta_g^t$. Substituting the same in the last term of regret, we get:
\begin{align*}
&\sum_{t \in T} \sum_{g\in G}\mathbbm{1} ( g_t=g ){R_g^t}  = \sum_{t \in T} \sum_{g\in G}\mathbbm{1} (g_t=g)(R_{g^*}^* - \Delta_g -\Delta_g^t)\\
&= TR_{g^*}^* - \sum_{g\in G}\mathbb{E}_\pi[N_{g,T}]\Delta_g - \sum_{t \in T} \sum_{g\in G}\mathbbm{1} (g_t=g)\Delta_g^t
\end{align*}
Substituting the same in the regret, we get: $\mathfrak{R}_\pi^\beta(T) = \sum_{g\in G}(\mathbb{E}_\pi [N_{g,T}] - \lfloor\beta_gT\rfloor)\Delta_g + \sum_t \sum_{g\in G} \mathbbm{1}(g_t = g)\Delta_g^t$.
\end{proof}
\noindent The first term in Equation (\ref{eq:RegretDecomposition}), i.e., $\sum_{g\in G}(\mathbb{E}_\pi [N_{g,T}] - \lfloor\beta_gT\rfloor)\Delta_g$, represents the cumulative regret due to extra number of times suboptimal group is pulled above the minimum guaranteed pulls $\lfloor\beta_gT\rfloor$ required to satisfy group-fairness constraints. The second term, $\sum_t \sum_{g\in G} \mathbbm{1}(g_t = g)\Delta_g^t$, represents the regret due to choosing a suboptimal policy for arm pulls within the group. For a group $g$, the optimal policy gives the expected reward of $R_g^*$, whereas choosing a policy $\pi^t$, gives the reward of $R_g^t$. We call this difference the regret due to choosing a non-optimal policy.

\subsection{Regret of \ouralgo}
\label{sec:analysis}

The regret of \ouralgo\ can be bounded by bounding each term separately. We now provide these bounds here with proofs referred to in the appendix.
\subsubsection*{Bounding Regret due to Sub-optimal group selection}
In order to bound this, we show that if we have pulled a sub-optimal group enough number of rounds, we will be able to distinguish the sub-optimal group from the optimal group with high probability and therefore, we will never select that group further. This leads to the following lemma.

\begin{restatable}{lemma}{firsttermbound}
Under Assumption 2,
\allowdisplaybreaks
\begin{align*}
&\sum_{g\in G} \mathbb{E}_\pi ([N_{g,T}] - \lfloor \beta_g T\rfloor)\Delta_g \le \left(1+\frac{\pi^2}{3}\right)\sum_{g\in G}\Delta_g\\ 
&\!\!\!\!\!+ \sum_{g\in G} \left(\frac{k_gf(\gamma_2)}{f(\gamma_1)}\left(\frac{8L_1^2}{(\Delta_{min})^2}\ln \left(\frac{4N_{g,T}k_g}{\delta}\right) + \sqrt{\frac{N_{g,T}\ln(k_g/\delta)}{2}}\right)-\beta_gT\right)\Delta_g
\end{align*}
Here, $\Delta_{min} = \min_{g\neq g^*} R_{g^*}^* - R_g^*$ denotes the minimum difference between expected reward between the optimal and sub-optimal group with known rewards. Here, $L_1$ is a Lipschitz's constant that satisfies $|R_g(\mu)-R_g(\mu')| \le L_1|\mu-\mu'|$. Lipschitz continuity on reward function follows from Lipschitz continuity of merit function $f(\cdot)$.
\end{restatable}

We provide the proof in the appendix which essentially follows similar steps to that of UCB by making use of a few additional results such as Lipschitz continuity on the reward function, minimum number of pulling an arm when a group is selected.  Once we have these results, we can prove that 
if a group is pulled sufficient number of times, then each arm in that group is also pulled sufficient number of times due to meritocratic fairness. Since the reward function is Lipschitz continuous, this leads to a distinction of sub-optimal group from the optimal group.

\subsubsection*{Bounding regret due to fairness of exposure within each group}

In order to bound the second term of the regret, the difference in policy is considered for the time periods when a group $g$ is selected. The proof follows similar steps as that in \cite{wang2021fairness} after replacing $T$ with $N_{g,t}$ (number of times a group $g$ is pulled till time $t$ in the confidence region). Thus, the second part of the regret is given by the following lemma.

\begin{restatable}{lemma}{secondtermbound}
\label{theo:GFIE_UCB_RR}
The second part of the regret is
$O\left(\sum_{g\in G}\sqrt{N_{g,T}k_g} \right)$ with probability at 
least $1-\delta$.  
\end{restatable}

Thus, combining the results above leads to the following bound on the reward regret.
\begin{theorem}
The reward regret (\groupReward) of \ouralgo\ is given as: 
\begin{align*}
&\mathfrak{R}_\pi^\beta (T)= \left(1+\frac{\pi^2}{3}\right)\sum_{g\in G}\Delta_g+ \sum_{g\in G}\sqrt{N_{g,T}k_g}(1-\delta) + \delta T\\
&\!\!\!\!\!+\sum_{g\in G} \left(\frac{k_gf(\gamma_2)}{f(\gamma_1)}\left(\frac{8L_1^2}{(\Delta_{min})^2}\ln \left(\frac{4N_{g,T}k_g}{\delta}\right) + \sqrt{\frac{N_{g,T}\ln(k_g/\delta)}{2}}\right)-\beta_gT\right)\Delta_g
\end{align*}
Substituting $\delta$ to be $\Omega(1/\sqrt{T})$, we get the regret of $O(\sqrt{T n})$.
\end{theorem}
\section{Experiments}
\begin{figure}
\hspace{-5mm}
\begin{subfigure}{0.24\textwidth}
\includegraphics[scale=0.107]{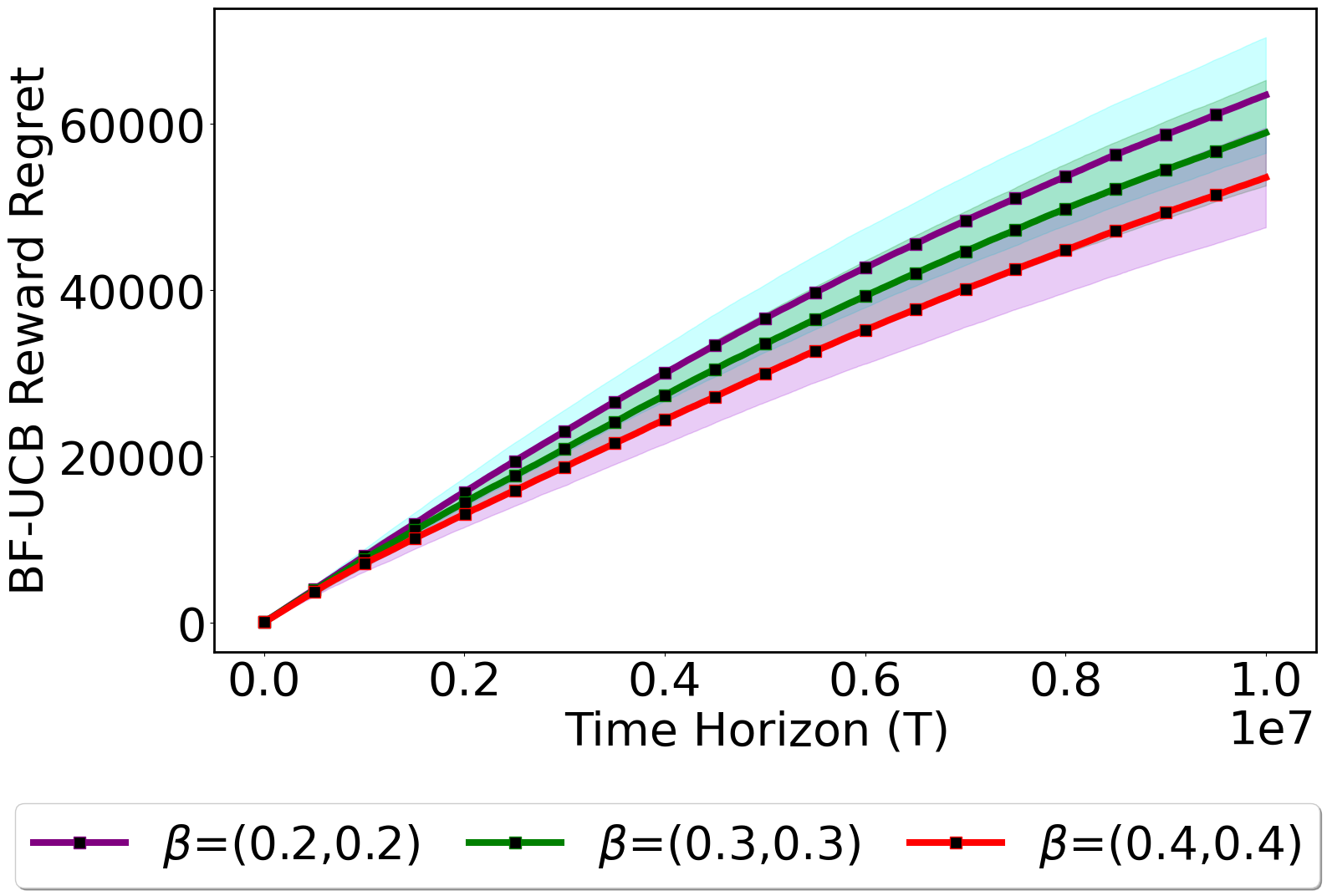}
\caption{For high number of arms}
\label{fig:11}
\end{subfigure}
\begin{subfigure}{0.24\textwidth}
\includegraphics[scale=0.107]{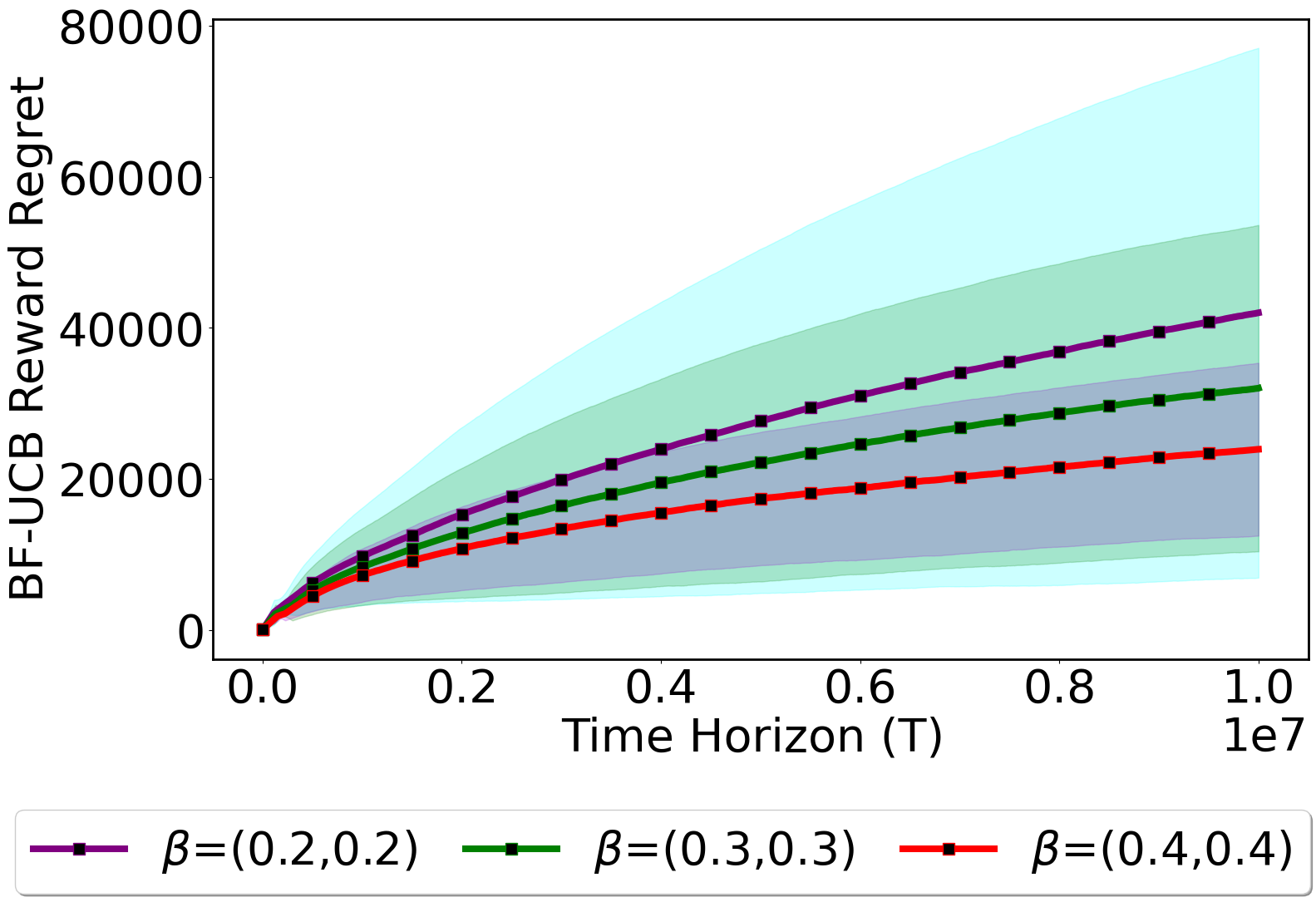}
\caption{For low number of arms}
\label{fig:31}
\end{subfigure}
\hspace{-5mm}
\caption{For the \ouralgo\ algorithm: Comparison of Reward Regret over time  for different values of $\beta$}
\end{figure}

\begin{figure}
\hspace{-5mm}
\begin{subfigure}{0.245\textwidth}
\includegraphics[scale=0.11]{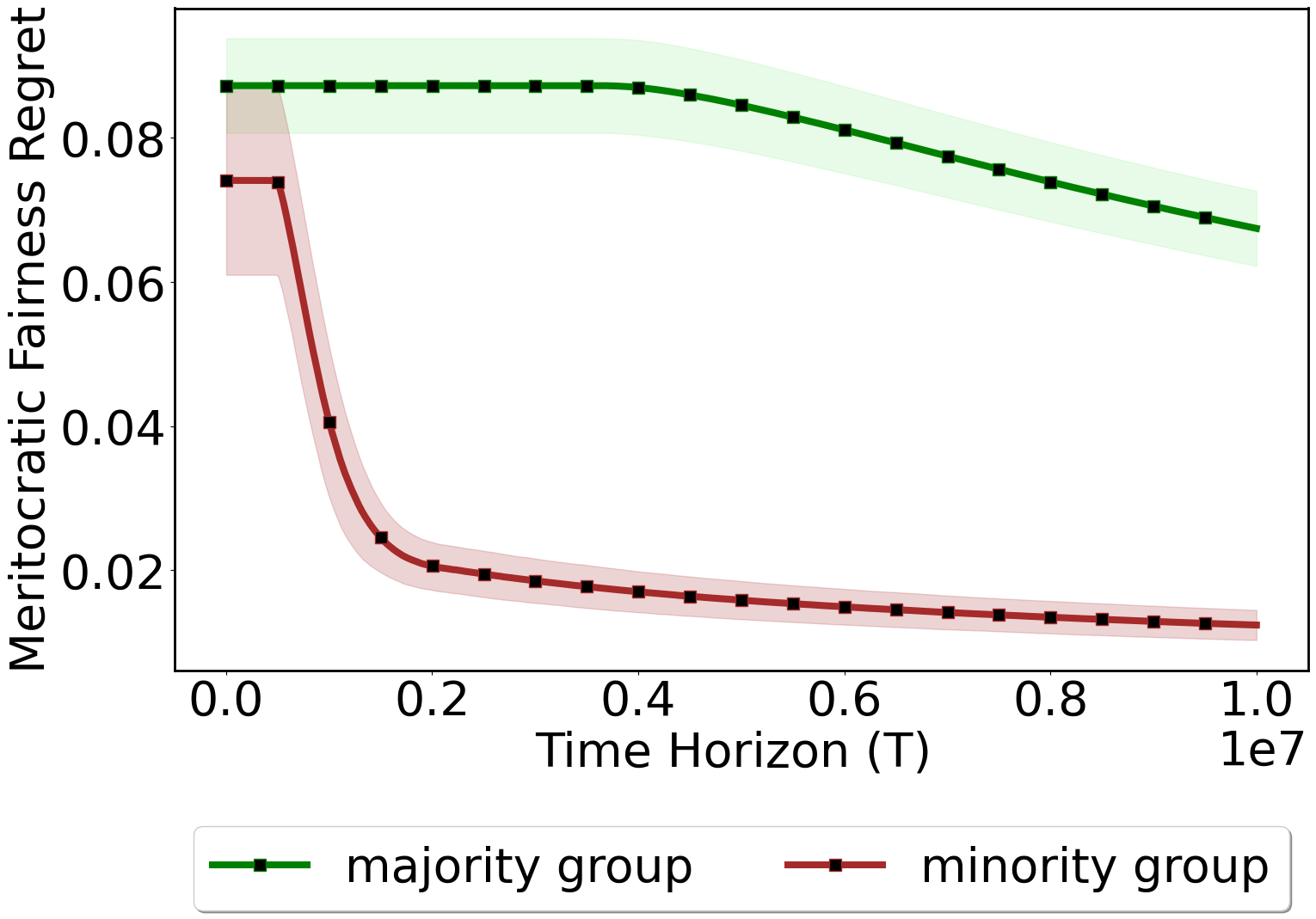}
\caption{For high number of arms}
\label{fig:12}
\end{subfigure}
\begin{subfigure}{0.245\textwidth}
\includegraphics[scale=0.11]{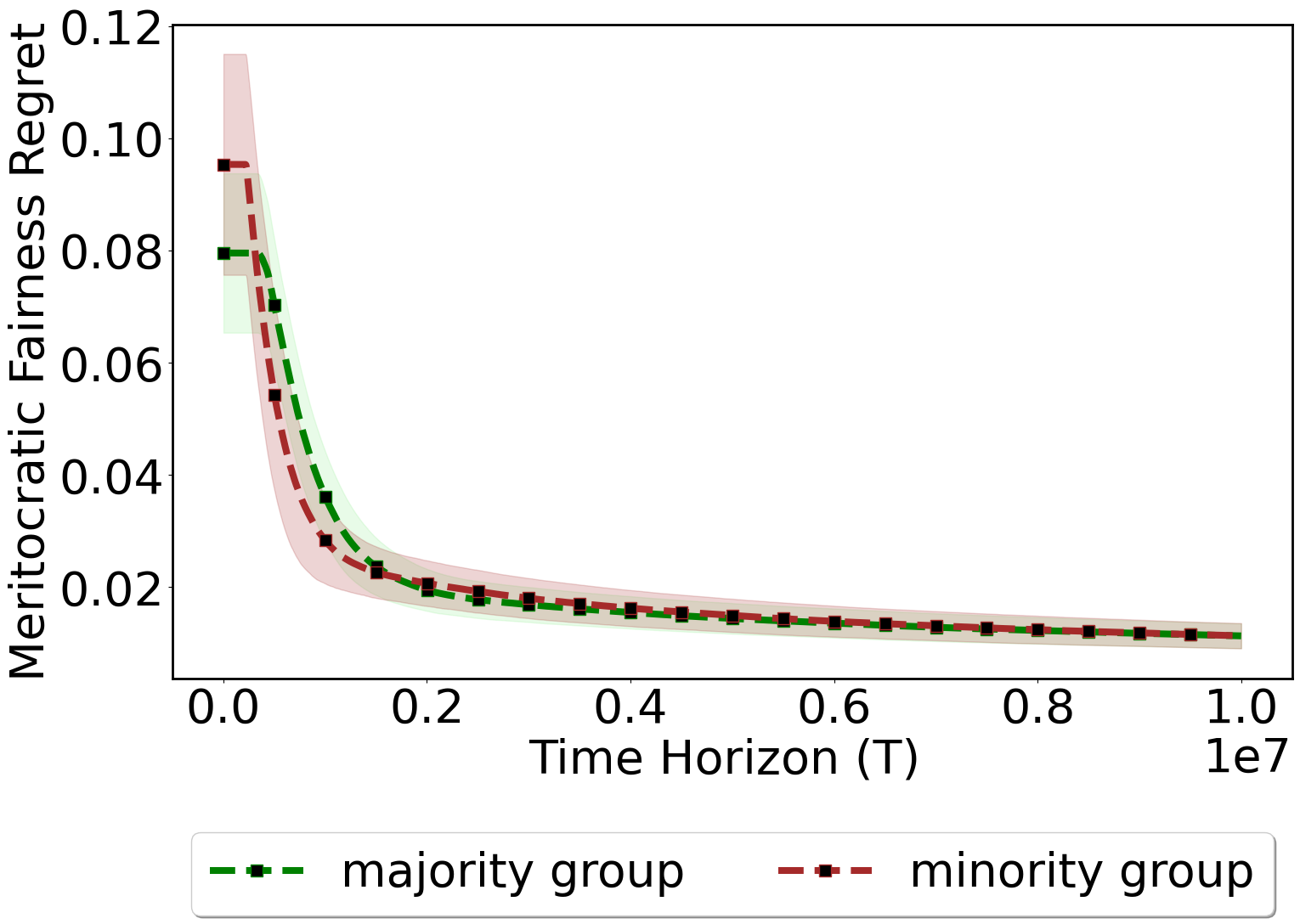}
\caption{For low number of arms}
\label{fig:32}
\end{subfigure}
\hspace{-5mm}
\caption{For the \ouralgo\ algorithm: Comparison of  Meritocratic Fairness Regret over time across the different groups}
\end{figure}
In this section, we analyze our algorithm for regret and fairness via simulated experiments. The goal is to study the effect of the number of arms on regret and fairness guarantees, and also, how (i) \mgfair\ and (ii) \mifair\ guarantees of \ouralgo\ compares with that of UCB~\cite{auer2002using},~\cite{patil2021achieving} and ~\cite{wang2021fairness}. 
We first start by explaining these baselines, followed by our experimental setup and results.
\subsection{Baselines}
\subsubsection{UCB}
This baseline is a conventional UCB algorithm \cite{lai1985asymptotically} that aims to maximize the total reward obtained by pulling any arm without any fairness constraints.
\subsubsection{Meritocratic Fair Algorithm (MF)}

The MF algorithm ensures meritocratic fairness across all arms independent of the groups \cite{wang2021fairness} in which they are present. 
\subsubsection{Group Exposure Fair Algorithm (GEF)}
 This algorithm is an adaption from \cite{patil2021achieving} to group exposure where when a group is chosen, the arm with the highest reward is preferred instead of ensuring meritocratic fairness within the group. 
\begin{figure*}
\begin{subfigure}{0.33\textwidth}
\includegraphics[scale=0.125]{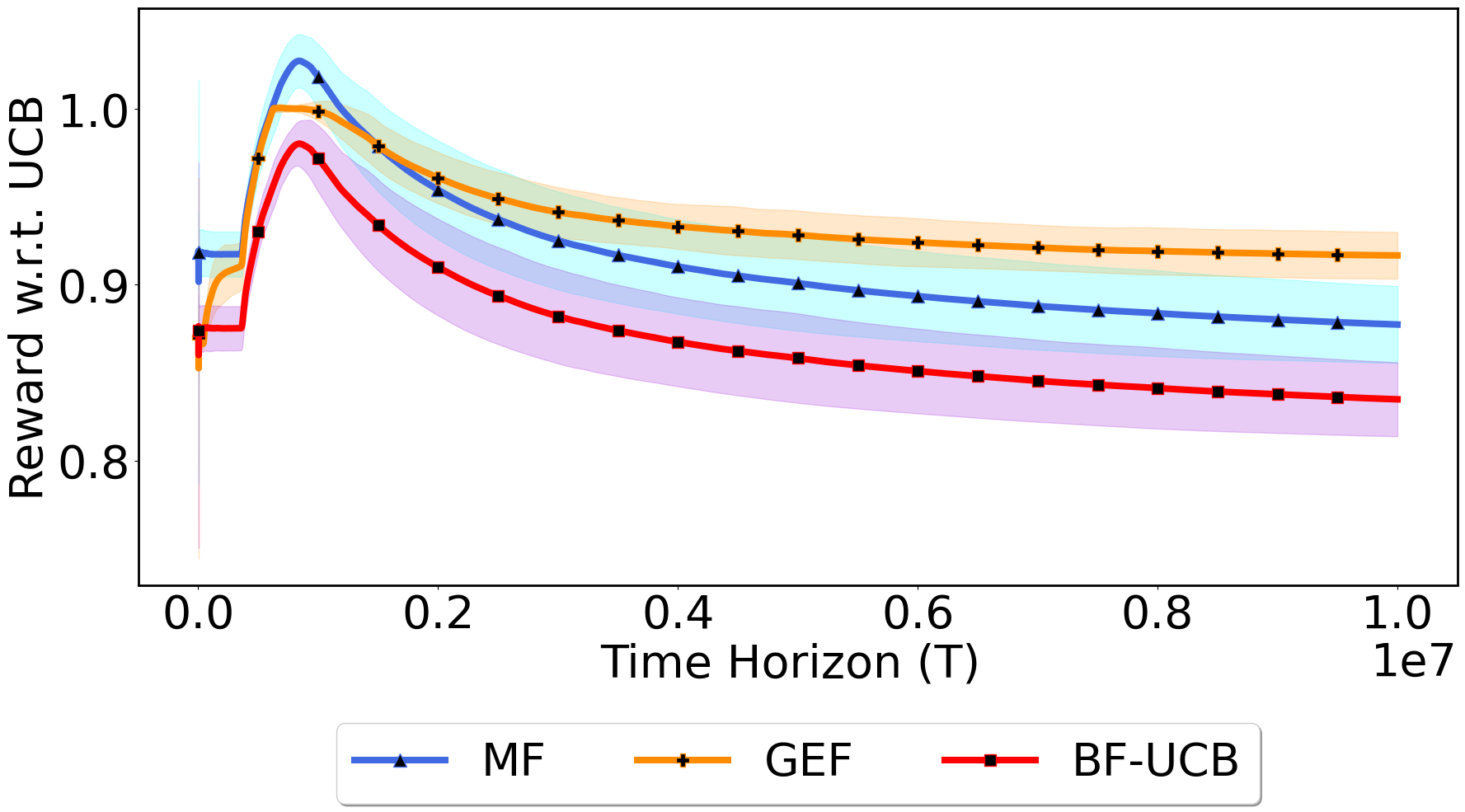}
\caption{Total Reward (w.r.t. UCB)}
\label{fig:21}
\end{subfigure}
\begin{subfigure}{0.30\textwidth}
\includegraphics[scale=0.12]{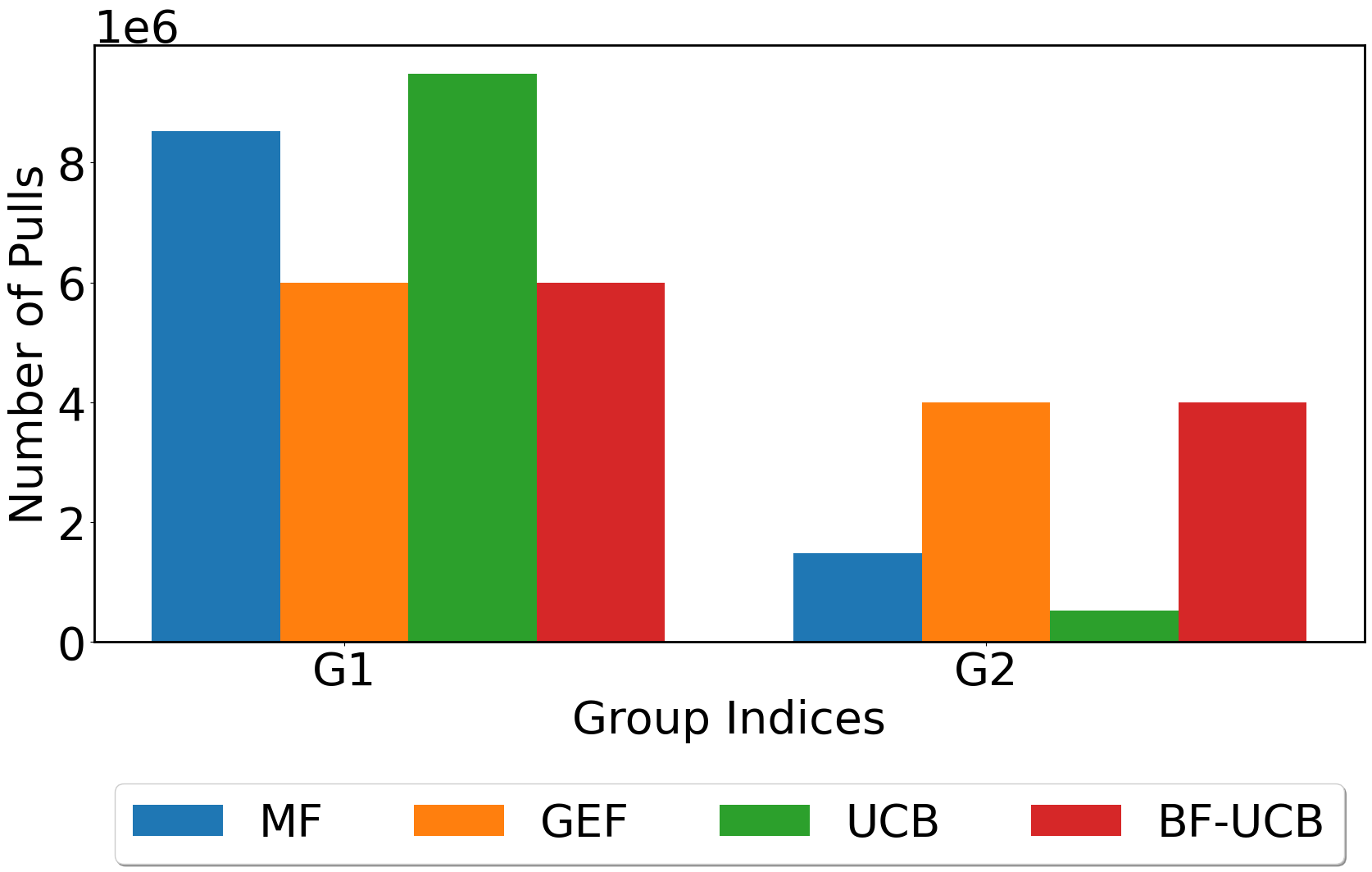}
\caption{Group Exposure}
\label{fig:22}
\end{subfigure}
\begin{subfigure}{0.35\textwidth}
\includegraphics[scale=0.12]{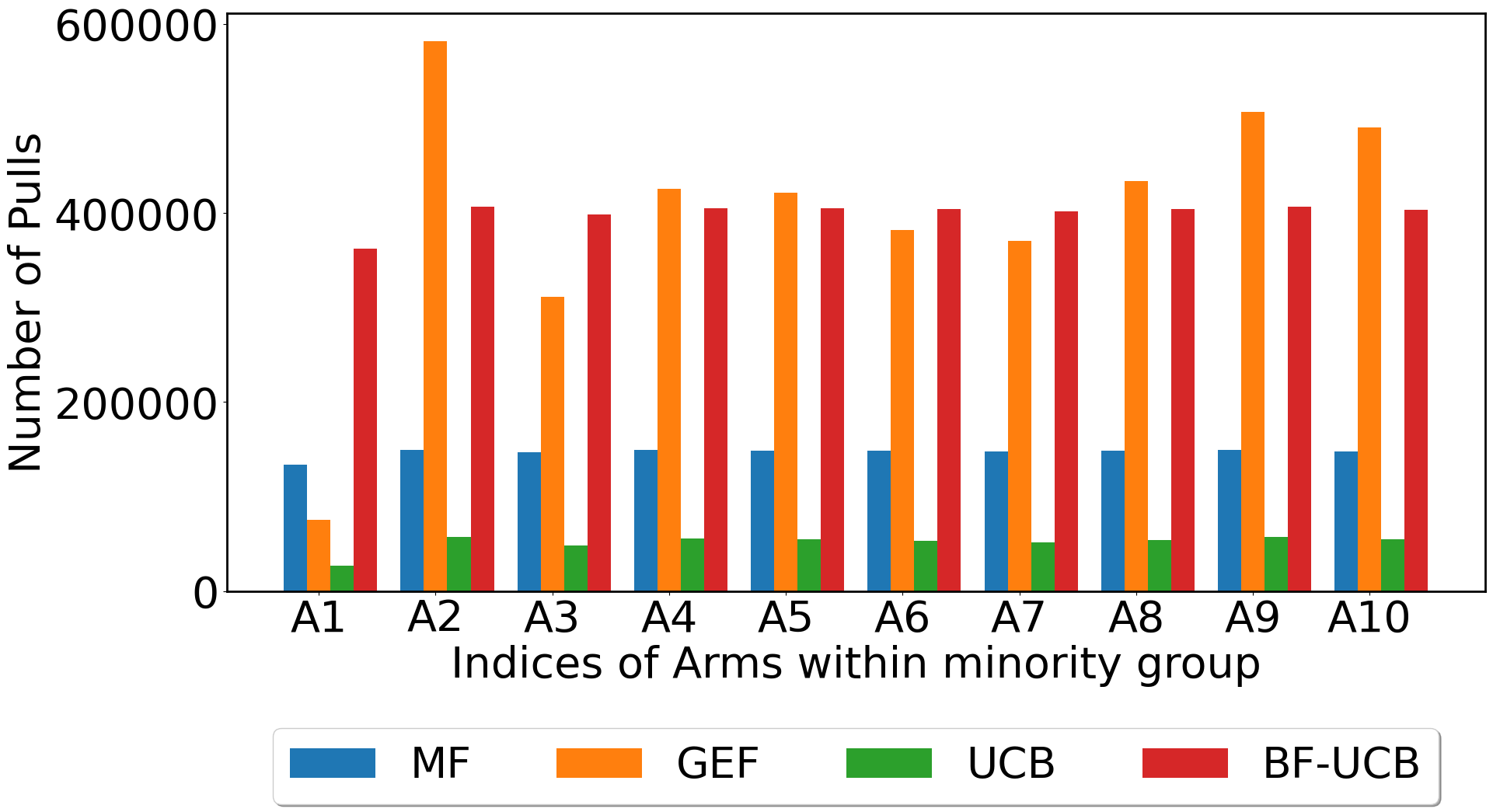}
\caption{\mbox{Exposure across Arms (within minority group)}}
\label{fig:23}
\end{subfigure}
\caption{Comparison of \ouralgo , \mgfair\ and \mifair\ on different
performance measures for the setting involving high number of arms}
\end{figure*}
\subsection{Experimental Setup}
We have considered the number of groups to be two, inline with the group fair literature where mostly majority and minority groups are considered. We have run our experiments for a total time $T = 10^7$,
and ran 50 random runs of each of the experiments to plot the results~\footnote{The code is available at: \url{https://github.com/MultiFair-Bandits/Stochastic_Fair_Bandits/}}.
In order to show the efficacy of our algorithm, we have considered two settings:
\begin{enumerate}[leftmargin=*]
\item {\em Low number of arms\/}: In this setting, we consider the number of arms in minority and majority groups to be five and ten, respectively. The mean rewards of the arms from both groups are generated uniformly from $[0.6, 0.85]$. In this setting, there is very little separability amongst the rewards of the arms, and thus, each run may lead to a different optimal group. 
The arm probabilities are generated afresh in each run.  

\item 
{\em High number of arms\/}: Here, the minority and majority groups contain ten and fifty arms, respectively. The mean rewards of arms from the majority and minority groups are generated uniformly from $[0.7, 1]$ and $[0.5, 0.8]$ respectively. This setting has clear separability amongst the optimal and sub-optimal group where majority group is optimal for all the rounds. 
\end{enumerate}

\subham{These threshold on number of arms is motivated by real-world examples such as the Adult dataset \cite{adult-dataset},  where a typical ratio between two group values (sensitive attribute race) is typically 1:8 and in gender  attributes, the typical ratio is 1:2.} 
We consider merit function $f(\mu) = \mu$ and $\delta=.01$. The merit function is chosen thus as it can be shown that the maximum value of the reward function with the above merit function is always achieved at the highest value of $\mu$ when all $\mu_i$'s are greater than $0.5$. The proof of this result is provided in the Appendix. Therefore, such a merit function allows us to directly use the upper confidence value of $\mu$ without explicitly computing the optimal value. It is also to be noted that the regrets will not be affected much by different merit functions. We now explain the results of \ouralgo\  on different performance measures in comparison with the baselines.
\subsection{Experimental Results}
For all the comparisons, we consider $\beta = (0.4,0.4)$ except for the comparison of regret, where we plot the total regret against all three different $\beta$ values, namely, $(0.2,0.2), (0.3,0.3),$ and $(0.4,0.4)$. 

\subsubsection{Reward Regret} 
Figures \ref{fig:11} and \ref{fig:31} show the reward regret for the two settings, namely, high and low number of arms, respectively, for different values of $\beta$. As only \ouralgo\ maintains \ourfair, the regret of only \ouralgo\ is plotted. It can be seen from both the figures that the regret is sub-linear. A higher value of $\beta$ puts more constraint on the group exposure guarantee, leading to lower regret due to the sub-optimal group pull. For instance, when $\beta = (0.5,0.5)$, both \ouralgo\ and the optimal algorithm will end up pulling both the groups in a round-robin fashion, thus leading to a regret of zero in the first term. The high variance for a lower number of arms setting is due to non-separability in rewards of the arms. This leads to a change in the optimal group over different runs, leading to high variance.

\subsubsection{\ifair\ Regret} 
Figures \ref{fig:12} and \ref{fig:32} show the policy regret for the different groups, i.e., $|\pi^*_g - \pi^t_g|$ in the two settings, respectively. It can be seen from the figures that policy regret eventually converges to zero. It should be noted that though one would expect the policy regret of the majority group, which is optimal in almost all cases, in Figure \ref{fig:12} to converge faster, we do not see such a trend here. This is primarily due to the large number of arms in the majority group, which makes it difficult to converge faster. On the other hand, when we have low number of arms, we see this convergence much faster in Figure \ref{fig:32}. 

\subsubsection{Total Reward} 
Figure \ref{fig:21} compares the total reward of \ouralgo\ with different baselines for the higher number of arms setting. The rewards of different baselines are normalized with respect to the reward of UCB. As can be seen from the figure, the rewards of different algorithms initially increase with respect to UCB and then decrease gradually with time. The initial increase is due to the exploration phase of all the algorithms leading to similar rewards in the initial rounds. After a few rounds, UCB will start picking the arm with maximum reward, whereas other algorithms will have to satisfy the fairness constraint and hence, they will receive a lesser reward as compared to UCB. Since \mgfair\ still picks the best arm in the group whereas \mifair\ has to ensure exposure fairness across all the arms, the reward of \mgfair\ is higher than that of \mifair.
It must be noted that the normalized rewards are not too far from $1$ and the difference in the rewards across various baselines is not much.
As expected, \ouralgo\ receives the least reward amongst all the algorithms as it needs to satisfy the strictest fairness notion.

\subsubsection{Group Exposure} 
Figure \ref{fig:22} compares the number of times each group is pulled across different algorithms for the higher number of arms setting. As can be seen, \ouralgo\ and \mgfair\ give the most balanced exposure to the two groups. UCB gives the least exposure. On the other hand, since \mifair\ provides the exposure guarantees across all arms, it still ends up pulling the majority group a significantly larger number of times as compared to the minority group. This figure shows that just ensuring exposure fairness across individual arms does not guarantee enough exposure to the groups.


\subsubsection{Exposure across Arms} 
Figure \ref{fig:23} plots the exposure of different arms only from the minority group for the higher number of arms setting. It shows that \mifair\ gives the least exposure to these arms, as there is a high number of arms in the majority group, thus leading to low exposure of arms in the minority group. The exposure to the arms is best when employing \ouralgo. \mgfair\ algorithm, though it seems to be giving good exposure, it should be noted that it has high variance because at each run, the optimal arm will be different and \mgfair\ aims to pull the optimal arm. UCB algorithm gives the least exposure to the arms present in the minority group. This figure shows that \ouralgo\ not only ensures group exposure but also ensures individual arm exposure within each group.


\iffalse
The results discussed in Figure 4. are Cumulative Reward over Time in which \ouralgo\ is compared with  (i.e \textbf{$\pi_{i,T}$}) within  group g, from the group policy (i.e $\pi_g$) averaged over 50 trials for $T=10^6$ rounds each. For comparison to showcase the importance of \ouralgo\ algorithm we have compared with the competing baselines i.e exposure Across group  based \gfair\ , Merit based \ifair\ and UCB as baselines.
Then we move towards our Initial Results, that are from Figure 1. where Regret Plot versus Time for different $\beta-$\gfair are mentioned in Fig 1a. with Group wise Policy Regret at each time with its mean and standard deviation are discussed in Fig 1b., Lastly, in Fig 1c. Group exposure is discussed with number of extra pulls after providing group exposure for each group w.rt to time is discussed. 
Next, In Experimental Results  We have discussed in Figure 1, \ouralgo\ cumulative Reward Regret at different $\beta$ are plotted w.r.t time in Figure 1a..
First, we have shown that \ouralgo\  results in sublinear Reward Regret with respect to time for 
multiple runs based varying distribution of reward allocation to the arm at different $\beta$. 
Second, In Fig 1b. and in Fig 1c. \ifair\ and \gfair\ among optimal and suboptimal groups at $\beta=[0.4,0.4]$ are analysed respectively.
Lastly, \ifair, \gfair, and UCB are compared with \ouralgo. In Fig 2a. The comparison are done on Cumulative Reward v/s time, In Fig 2b. comparison for $\beta-$\mgfair\ is analysed at $\beta$=[0.4,.4] by plotting Group Pulls over corresponds to optimal and suboptimal group, and In Fig 2c. \mifair\ is analysed for within suboptimal group \ifair\ to arms.\\
We have done fair comparisons  to describe where our merit lies while implementing Bi-level fairness, We have done comparison of \ouralgo with \ifair, \gfair, and UCB.  Before discussing the performance and analysis we have discussed the competitive baselines in detail:\\

\subsection{Reward Regret sublinear convergence w.r.t Time}
From experimental evaluation shown in Fig1. we observes \ouralgo\  Reward Regret w.r.t time to be sublinear. The result are obtained on different distribution. The Cumulative Reward Regret with mean and standard deviation for $\beta-$\mgfair\ is [.2,.2], [.3,.3] and [.4,.4] (for 2 groups), \ouralgo\ shows sublinear convergence.
\subsection{Meritocratic Fairness within Groups }
\ifair\  is observed from Figure 2a. w.r.t time that guarantees \mifair\ is ensured for each group $g \in G$ at  time t ,. Figure 2a. plots for $\sum_{i\in g}|\pi_i^*-\pi_i^t|$ w.r.t time t $\forall\; g\; \in\; G$ converges  for large enough t where $T=10^7$. For the experimental setup we have used $\beta-$\mgfair\ to be [0.4,0.4] and $\mu^i_{g1} =(0.7,1)$ and $\mu^i_{g2} =(0.5 ,0.8)$, the optimal and suboptimal group have number of arms 50 and 10 respectively.
\ouralgo\ ensures each arm within group based on merit policy $\pi_i^t$  converges to the merit within groups an optimal policy ($\pi_g^*$) for large enough time of horizon $T$, i.e  $\forall i\in g $ $\forall g \in G$.
converges towards group optimal policy $\pi_g^*$.
\\
For each group $g\in G$ learner learns a merit policy i.e $\pi_{g,t}$. The policy learnt by the learner converges to the optimal policy  $\pi^*_{g}$ as discussed in Figure 2. \\
\subsection{Group Exposure Fairness}
 \gfair\ from Figure 2b. for  $\beta-$\mgfair\ =[0.4,0.4] \ouralgo\ ensures
 $N_{g,t}- \lfloor\beta_{g}t\rfloor \geq 0\;\forall\;g\;\in\;G\;$. Figure 3. have demonstrated experimentally that anytime group exposure is provided for suboptimal and optimal group while keeping least pulls for suboptimal group, and as time increases the number of pulls for optimal group is advanced by \ouralgo.

\subsection{Comparison with Baselines }
\textbf{Results From fig 4 5 6} \\ 
The results discussed in Figure 4. are Cumulative Reward over Time in which \ouralgo\ is compared with  (i.e \textbf{$\pi_{i,T}$}) within  group g, from the group policy (i.e $\pi_g$) averaged over 50 trials for $T=10^7$ rounds each.
\\
In Figure 5. and 6. Comparison is shown based on average merit policy by each arm within each group g for \ouralgo\ algorithm .
The comparison of average merit policy for arm $i \in g$ for each group g, the \ouralgo\ algorithm is compared with proportionality based Group Fair Algorithm, Exposure Fair Algorithm and with benchmark optimal policy ($\pi^*_g$) for each group.
$\beta= [{.4,.4}]$ for different distribution trials,to compare \ouralgo\ with group fair and Fair exposure algorithm. \\
Figure 3(a) \ouralgo\ competing Cumulative Reward over time in comparison with UCB, \ifair and \gfair. \\
Figure 3(b) followed by the plot on Group Pulls v/s Time for competing baselines algorithm with $\beta-$\mgfair\  with its mean and variances for different distribution over arms. \\
Lastly, in Figure 3(c) Plot of Within Group Pulls v/s Time for suboptimal group has shown to achieve \mifair with its mean and variances for different distribution over arms.\\
\fi


\section{Conclusion}


  In summary, our novel fair Multi-Armed Bandit (MAB) framework, \ouralgo, ensures both \ifair\ and \gfair. Through rigorous regret decomposition analysis and from \ourfair\ guarantee, we established its theoretical foundation. Our experimental results demonstrated competitiveness in achieving normalized rewards relative to UCB, in comparison to MF and GF. We also showcased its practical utility in achieving fair exposure to the arms within minority groups. In conclusion, our \ourfair\ MAB algorithm, \ouralgo , is the first to give a robust solution for achieving \ourfair\  with sublinear regret.
\section*{Acknowledgement}
The research is supported by the Department of Science \& Technology, India, with grant number CRG/2022/007927.
\bibliographystyle{ACM_format}
\balance
\bibliography{aamas24}

\newpage

\appendix 
\section{Missing Proofs}
\subsection{Bounding First Term of Regret}

Let us define $R_g(\mu) = \sum_{a\in g}\frac{f(\mu_a)}{\sum_{a\in g}f(\mu_a)}\mu_a$ as the expected reward from a group $g$ with reward vector $\mu$. We then have the following Lemma, which directly follows from Assumption 1.

\begin{lemma}
If $f(\cdot)$ satisfies Assumption 1, then there exists a constant $L_1$ such that:
$$|R_g(\mu) - R_g(\mu')| \le L_1|\mu_a - \mu_a'| \forall a \in g$$
\end{lemma}
\begin{proof}
 To prove the above claim, we will prove that $R_g$ is locally Lipschitz $\forall \;\mu$s satisfying Assumption 2.  This would imply that $|R_g(f(\mu)) - R_g(f(\mu'))| \le L'|f(\mu) - f(\mu')|$. Once, we are able to prove that the above lemma will directly follow from the Lipschitz continuity of $f$. For claim that, $R_g$ is locally Lipschitz as a function of $f(\mu)$, consider its gradient w.r.t. $f(\mu_i), \nabla_i R_g = \frac{\sum_j f(\mu_j)\cdot \mu_i - \mu_i\cdot f(\mu_i)}{(\sum_j f(\mu_j))^2} \leq \frac{\gamma_2^2k_g\max_i \mu_i}{k_g^2\gamma_1^2 } = \frac{\gamma_2^2\max_i \mu_i}{k_g\gamma_1^2}$. Note that $0<\gamma_1\leq f(\mu_i)\leq \gamma_2 , \forall i$. Since the gradient is bounded, this naturally means that $R_g$ is locally Lipschitz, satisfying the claim.
\end{proof}

We have the following lemma, which follows from Assumption 2 by applying upper and lower bounds.

\begin{lemma}
For any arm $a$ in group $g$, the probability of pulling the arm $a$ is lower bounded by $\frac{\gamma_1}{k_g\gamma_2}$. 
\end{lemma}

We next have the following lemma that provides the lower bound on the number of arms pulled in each group. This proof directly follows from Hoeffding's bound \cite{hoeffding1963prob}.

\begin{lemma}
\label{lemma:bounding-ucb-lcb}
If each group is pulled $N_{g,t}$ number of rounds, then we have $N_{a,t} \ge \frac{N_{g,t}f(\gamma_1)}{k_gf(\gamma_2)} - \sqrt{\frac{N_{g,t}\ln(k_g/\delta)}{2}}$ with probability atleast $1-\delta$ for all arms $a \in g$.
\end{lemma}
\begin{proof}
From Hoeffding's inequality, if $X_1, X_2, \ldots, X_n$ are independent random variables with $0 \le X_i \le 1$ then we have $$P(E[S_n] - S_n \ge \epsilon) \le exp\left(\frac{-2\epsilon^2}{n}\right)$$
Here, $S_n$ is the sum of all the $X_i$'s. For each arm $a$, assume $X_i$ denote the random variable if the arm is pulled or not. Then, substituting $\epsilon = \sqrt{\frac{N_{g,t}\ln(k_g/\delta)}{2}}$ in the above Hoeffding's inequality we have:
$$N_{a,T} \ge \frac{N_{g,t}f(\gamma_1)}{k_gf(\gamma_2)} - \sqrt{\frac{N_{g,t}\ln(k_g/\delta)}{2}}$$ with probability atleast $1-\frac{\delta}{k_g}$. Here, the inequality follows from the fact that $\mathbb{E}[S_n] \ge \frac{N_{g,t}f(\gamma_1)}{k_gf(\gamma_2)}$ from assumption 2.
Applying union bound, we get the required result.
\end{proof}

Let us also define $\bar{\mu}_t^g = arg\min_{\mu \in CR_t} \sum_{i\in g}\frac{f(\mu_i)}{\sum_{i'\in g}f(\mu_{i'})}\mu_i$. Then, we have the following result again following through Hoeffding's bound.

\begin{lemma}
\label{lemma:hoeffdings-mu}
At any time $t$, $\mathbb{P}(R_{g^*}(\Tilde{\mu}_t^{g^*}) \le R_{g^*}^{g^*}) \le \frac{\delta^2}{8N_{g,t}^2k_g}$ and $\mathbb{P}(R_{g^*}(\bar{\mu}_t^{g^*}) \ge R_{g^*}^{g^*}) \le \frac{\delta^2}{8N_{g,t}^2k_g}$. 
\end{lemma}
\begin{proof}
\begin{align*}
\mathbb{P}(R_{g^*}(\Tilde{\mu}_t^{g^*}) \le R_{g^*}^{g^*}) &= \mathbb{P}(\mu^g \notin [\hat{\mu}^g - w_t, \hat{\mu}_g + w_t])\\
&\le \sum_{i\in g} \mathbb{P}(\mu_i \notin [\hat{\mu}_{i,t} - w_{i,t}, \hat{\mu}_g + w_{i,t}])\\
&\le \sum_{i\in g} \mathbb{P}(|\hat{\mu}_{i,t} - \mu_i| \ge w_{i,t})\\
&\le \sum_{i\in g} 2exp\{-w_{i,t}^2N_{i,t}\}\\
&\le \frac{\delta^2}{8N_{g,t}^2k_g}
\end{align*}
\end{proof}

Next, we define $\Delta_{min} = \min_g R_{g^*}^* - R_g^*$.  Then, we have the following lemma.

\begin{lemma}
If $N_{i,t} \ge \frac{8L_1^2\ln\left(\frac{4N_{g,T}k_g}{\delta}\right)}{(\Delta_{min})^2}\ \forall i\in g$, then we have $R_g(\Tilde{\mu}_t^g) - R_g(\bar{\mu}_t^g) \le \Delta_{min}$.
\end{lemma}
\begin{proof}
 We have:
    \begin{align*}
    |\Tilde{\mu}_{i,t}^g - \bar{\mu}_{i,t}^g| &\le |\Tilde{\mu}_{i,t}^g - \bar{\mu}_{i,t}^g + \mu_i - \mu_i|\\
    &\le |\Tilde{\mu}_{i,t}^g - \mu_i| + |\bar{\mu}_{i,t}^g - \mu_i|
    &\le 2w_{i,t} = 2\sqrt{\frac{2\ln\left(\frac{4N_{g,t}k_g}{\delta}\right)}{N_{i,t}}}\\
    &\le 2\sqrt{\frac{2\ln\left(\frac{4N_{g,T}k_g}{\delta}\right)}{N_{i,t}}}\\
    &\le \frac{\Delta_{min}}{L_1}
    \end{align*}
    Thus, $R_g(\Tilde{\mu}_t^g) - R_g(\bar{\mu}_t^g) \le L_1|{\mu}_t^g - \bar{\mu}_t^g| \le  \Delta_{min}$.
\end{proof}

\begin{lemma} 
If $N_{g,t} \ge \frac{k_gf(\gamma_2)}{f(\gamma_1)}\left(\frac{8L_1^2}{(\Delta_{min})^2}\ln \left(\frac{4N_{g,T}k_g}{\delta}\right) + \sqrt{\frac{N_{g,T}\ln(k_g/\delta)}{2}}\right)$, then $N_{i,t} \ge \frac{8L_1^2\ln\left(\frac{4N_{g,T}k_g}{\delta}\right)}{(\Delta_{min})^2}$ with probability $1-\delta$.
\end{lemma}

\begin{proof}
From Lemma 7, we have, 
\begin{align*}
N_{i,t} &\ge\frac{N_{g,t}f(\gamma_1)}{k_gf(\gamma_2)} - \sqrt{\frac{N_{g,t}\ln(k_g/\delta)}{2}}\\
&\ge \frac{N_{g,t}f(\gamma_1)}{k_gf(\gamma_2)} - \sqrt{\frac{N_{g,T}\ln(k_g/\delta)}{2}}\\
&\ge \frac{8L_1^2\ln\left(\frac{4N_{g,T}k_g}{\delta}\right)}{(\Delta_{min})^2}
\end{align*}
\end{proof}

\firsttermbound*


\begin{proof}
Denote the group selected at time $t$ by the algorithm as $g_t$. We then want to bound the rounds in which suboptimal group $g \ne g^*$ is pulled. Then for the algorithm, we have:
$$N_{g,T} = t_{init} + \sum_{t=t_{init}+1}^T \mathbbm{1}\{g_t = g\}$$
More generally, we can write it as:
$$N_{g,T} = l + \sum_{t=l}^T \mathbbm{1}\{g_t = g, N_{g,t-1} \ge l\}$$
If $g_t = g$ then $R_{g^*}(\Tilde{\mu}_t^{g^*}) \le R_{g}(\Tilde{\mu}_t^g)$. Thus,
\begin{align*}
N_{g,T} &\le l + \sum_{t=l}^T \mathbbm{1}\{R_{g^*}(\Tilde{\mu}_t^{g^*}) \le R_{g}(\Tilde{\mu}_t^g), N_{i,t-1} \ge l\}\\
&\le l + \sum_{t=l}^T \mathbbm{1}\left\{\min_{0< s_{g^*} < t} R_{g^*}(\Tilde{\mu}_{s_{g^*}}^{g^*}) \le \max_{l \le s_i < t} R_{g}(\Tilde{\mu}_{s_i}^g), N_{i,t-1} \ge l\right\}
\end{align*}
 Then, it is easy to see that if at time $t$, $R_{g^*}(\Tilde{\mu}_t^{g^*}) \le R_{g}(\Tilde{\mu}_t^g)$, then one of the following has to be true:
\begin{itemize}
\item $R_{g^*}(\Tilde{\mu}_t^{g^*}) \le R_{g^*}^*$
\item $R_g(\bar{\mu}_t^{g}) \ge R_g^*$
\item $R_{g^*}^* \le R_g^* + c_t$. Here, $c_t$ is a time-dependent constant which satisfies $R_g(\bar{\mu}_t^g) + c_t \ge R_g(\Tilde{\mu}_t^g)$ 
\end{itemize}

The following is easy to see from Hoeffding's inequality:
\begin{itemize}
\item $R_{g^*}(\Tilde{\mu}_t^{g^*}) > R_{g^*}^*$ and $R_g(\bar{\mu}_t^{g}) < R_g^*$ with high probability. This is bounded by Lemma \ref{lemma:bounding-ucb-lcb}.
\item If group $g$ is pulled atleast\\ $l = \frac{k_gf(\gamma_2)}{f(\gamma_1)}\left(\frac{8}{(F^{-1}(\Delta_{min}))^2}\ln \left(\frac{4N_{g,T}k_g}{\delta}\right) + \sqrt{\frac{N_{g,T}\ln(k_g/\delta)}{2}}\right)$ times, then $c_t \le \Delta_{min}\ \forall t\ge l$. Therefore, this will happen with probability $0$.
\end{itemize}
\begin{align*}
N_{g,T} &\le  \ l + \sum_{t=l}^T\sum_{s_{g^*}=1}^t \sum_{s_g = l}^t \mathbbm{1}\left\{R_{g^*}(\Tilde{\mu}_t^{g^*}) > R_{g^*}^*, R_g(\bar{\mu}_t^{g}) < R_g^*\right\}\\
&\le \left(\frac{k_gf(\gamma_2)}{f(\gamma_1)}\left(\frac{8}{(F^{-1}(\Delta_{min}))^2}\ln \left(\frac{4N_{g,T}k_g}{\delta}\right) + \sqrt{\frac{N_{g,T}\ln(k_g/\delta)}{2}}\right)\right)\\ 
&+ 1 + \frac{\pi^2}{3}
\end{align*}
\end{proof}

\subsection{Bounding Second Term of Regret}
In order to bound the second term of the regret i.e. $\sum_t \sum_{g\in G} \mathbbm{1}(g_t = g)\Delta_g^t$, we bound the exposure regret due to a group $g$ when it was selected. In order to do that, let us consider a set $T_g = \{t_1, t_2, \ldots, t_{N_{g,T}}\}$ which denotes the time steps when the group $g$ is pulled. We begin with the following lemma which is a direct consequence of Lemma~\ref{lemma:hoeffdings-mu}.
\begin{lemma}
\label{lemma:GFIE_group_fairness_mean_bound} We have, $\forall$ $t>t_{k_g},i\in[g]$, $\mathbb{P}(\Tilde{\mu}^g_t\in CR_t) \ge 1-\frac{\delta^2}{8N_{g,t}^2k_g}$. 
\end{lemma}

\begin{lemma}
\label{lemma:GFIE_group_fairness_concentration_width}
For any $\delta\in(0,1)$, with probability at least $1-\delta/2$,
$$
\left\lvert
\sum_{t\in T_g}\mathbb{E}_{i \in \pi_t^g}\sqrt{1/N_{i,t}} - \sum_{t \in T_g}\sqrt{1/N_{i_t,t}}
\right\rvert
\leq \sqrt{2N_{g,t}\ln(4/\delta)}. 
$$
\end{lemma}
\begin{proof}
The sequence 
$$\sqrt{1/N_{i_t,t}}-
\mathbb{E}_{i \in \pi_t^g}\sqrt{1/N_{i,t}}$$
is a martingale difference sequence $\forall t>t_{k_g}$
$$\left\lvert\sqrt{1/N_{i_t,t}}-
\mathbb{E}_{i \in \pi_t^g}\sqrt{1/N_{i,t}}
\right\rvert\leq 1$$
We can apply the Azuma-Hoeffding's inequality to get with probability at least $1-\delta/2$,$$
\left\lvert
\sum_{t\in T_g}\mathbb{E}_{i \in \pi_t^g}\sqrt{1/N_{i,t}} - \sum_{t\in T_g}\sqrt{1/N_{i_t,t}}
\right\rvert
\leq \sqrt{2N_{g,t}\ln(4/\delta)}. 
$$
Hence, the lemma is proved to be true.
\end{proof}

\secondtermbound*
\begin{proof}
The second part of the regret for a group $g$ is given as:

\begin{align*}
&=\left( \sum_{t\in T_g}\sum_i(\pi_g^{*}(i) -\pi_t^g(i))\mu_i \right) \\
&\leq \left( 2k_g + \sum_{t\in T_g, t\ge t_{k_g}}\sum_{i\in g} \pi_t^g(i) \Tilde\mu_{i,t}-\pi_t^g(i)\mu_i \right) \tag{with probability at least $1-\frac{\delta^2}{8N_{g,t}^2k_g}$}\\
&= \left( 2k_g + \sum_{t\in T_g, t\ge t_{k_g}}\sum_{i\in g} \pi_t^g(i)(\Tilde\mu_{i,t}-\hat\mu_{i,t}+\hat\mu_{i,t}-\mu_i) \right)\\
\end{align*}
\; Since\; $N_{g,t} \leq N_{g,T}$ \; for\; $T \geq t$,\; we have:\\
\begin{align*}
&\leq \left( 2k_g + \sum_{t\in T_g, t\ge t_{k_g}}\sum_{i\in g} \pi_t^g(i)2\sqrt{ \frac{2\ln(4k_gN_{g,t}/\delta)}{N_{i,t}}} \right)\\
&\leq \left( 2k_g + 2\sqrt{2\ln(4k_gN_{g,T}/\delta)}\sum_{t\in T_g, t\ge t_{k_g}}\mathbb{E}_{i \in \pi_t}\sqrt{ \frac{1}{N_{i,t}}} \right)\\
&\leq \left( 2k_g + 2\sqrt{2\ln(4k_gN_{g,T}/\delta)}\left(\sqrt{2N_{g,T}\ln(4/\delta)}+ \sum_{t\in T_g, t\ge t_{k_g}} \sqrt{ \frac{1}{N_{i_t,t}}}\right) \right)\\
&\leq \left( 2k_g + 2\sqrt{2\ln(4k_gN{g,T}/\delta)}\left(\sqrt{2N_{g,T}\ln(4/\delta)} +  2\sqrt{N_{g,T}k_g}\right) \right)\\
\end{align*}
The first inequality comes from Algorithm {\UCBGFc}, the second inequality comes from "confidence region" Lemma ~\ref{lemma:GFIE_group_fairness_mean_bound}, the third inequality comes from  "concentration width" Lemma~\ref{lemma:GFIE_group_fairness_concentration_width}, and the last inequality applies the AM-GM inequality. Thus, when $T>k_g$, we have that with probability at least $1-\delta$, 
\[
\mathfrak{R}_{\pi}^{\beta}(T) = \widetilde{O}\left(  \sqrt{N_{g,T}k_g}\right).
\]
This concludes the proof. 
\end{proof}

\section{Choice of Merit Function in Experiments}
The below lemma depicts that when $f(\mu_i) = \mu_i$, then $R_g(\mu)$ is maximized with maximum value of $\mu_i$ for all $i\in g$.
\begin{lemma}
$\max_{\mu_i\in[0.5,b_i]}\sum_{i\in g}\frac{\mu_i^2}{\sum_{j\in g}\mu_i} = \sum_{i\in g}\frac{b_i^2}{\sum_{j\in g}b_j}$.
\end{lemma}
\begin{proof}
Let $\mu_{-i}$ represent the value of $\mu$'s for all arms in group $g$, except $i$. We will prove that $R_g(b_i, \mu_{-i}) \ge R_g(\mu_i, \mu_{-i})\ \forall \mu_i \in [a_i,b_i], \forall \mu_{-i}$. This will immediately imply the lemma. We have:
\begin{align*}
R_g(\mu_i+a, \mu_{-i}) &= \frac{(\mu_i + a)^2 + \sum_{j\ne i}\mu_j^2}{\sum_{j\in g} \mu_j + a}\\
&= \frac{x+y}{z+a}
\end{align*}
Here, 
$x = \sum_{j\in g}\mu_j^2$, $y = a(2\mu_i + a)$, and $z = \sum_{j\in g} \mu_j$. Then, we have the following claim:
\begin{claim}
If $a + 2\mu_1 > 1$, then $\frac{x+y}{z+a} > \frac{x}{z}$.
\end{claim}
The above claim is true because $\frac{x + a(2\mu_i + a)}{z+a} > \frac{x+a}{z+a} > \frac{x}{z}$. The above claim is true because $z > x$. Thus, we have $R_g(\mu_i +a, \mu_{-i}) \ge R_g(\mu_i, \mu_{-i})\ \forall a$.

\end{proof}
\newpage

\end{document}